%% file: tvso.tex
\documentclass{article}

\usepackage[accepted]{icml2026}

\input{math_commands.tex}

\definecolor{aaltoBlack}{RGB}{0,0,0}%
\definecolor{aaltoGray}{RGB}{146,139,129}%
\definecolor{aaltoRed}{RGB}{237,41,57}%
\definecolor{aaltoBlue}{RGB}{0,101,189}%
\definecolor{aaltoYellow}{RGB}{254,203,0}%
\definecolor{aaltoPurple}{RGB}{102,57,183}%
\definecolor{aaltoTurquoise}{RGB}{0,168,180}%
\definecolor{aaltoGreen}{RGB}{0,155,58}%
\definecolor{aaltoLightGreen}{RGB}{105,190,40}%
\definecolor{aaltoOrange}{RGB}{255,121,0}%
\definecolor{aaltoFuchsia}{RGB}{177,5,157}%
\usepackage{url}
\usepackage{microtype}
\usepackage{graphicx}
\usepackage{booktabs}
\usepackage{hyperref}
\usepackage{lipsum} 
\usepackage{wrapfig}
\usepackage{subcaption}

\usepackage{hyperref}
\usepackage{url}
\usepackage{amsmath}
\usepackage{amsfonts}
\usepackage{graphicx}
\usepackage{textcomp}
\usepackage{booktabs}
\usepackage{xcolor} 
\usepackage{epsfig}
\usepackage{tabularx}
\usepackage[ruled,vlined]{algorithm2e}
\usepackage{algpseudocode}
\usepackage{enumitem}
\usepackage{lipsum}
\usepackage{bbm}
\usepackage{relsize}
\usepackage{multicol}
\usepackage{multirow}
\usepackage{wrapfig}
\usepackage{mathtools}
\usepackage{amsthm}
\usepackage[capitalize]{cleveref}
\usepackage[hang,flushmargin]{footmisc}

\usepackage[most]{tcolorbox}
\usepackage[scaled=.95]{inconsolata}

\definecolor{darkgreen}{rgb}{0.00, 0.39, 0.00}
\definecolor{darkred}{rgb}{0.55, 0.00, 0.00}
\definecolor{darkblue}{rgb}{0.00, 0.00, 0.8}
\definecolor{mycyan}{RGB}{43,166,166}
\tcbset{
	graybox/.style={
		colback=gray!7,       %
		colframe=gray!7,      %
		boxrule=0pt,           %
		arc=0pt,               %
		left=4pt, right=4pt,   %
		top=4pt, bottom=4pt    %
	}
}

\makeatletter
    \patchcmd\caption@subtypehook{\let\label\subcaption@label}%
	{\let\label\subcaption@label\let\ltx@label\subcaption@label}{}{\fail}
\makeatother

\theoremstyle{plain}
\newtheorem{theorem}{Theorem}[section]
\newtheorem{proposition}[theorem]{Proposition}

\theoremstyle{definition}

\newtheorem{remark}[theorem]{Remark}

\newcommand{\fakeparagraph}[1]{\vspace{.5em}\noindent\textbf{#1}.\hspace{0.25em}}

\SetKwFor{For}{for}{}{endfor}

\crefname{figure}{Fig.}{Figs.}
\Crefname{figure}{Fig.}{Figs.}
\crefname{table}{Tab.}{Tabs.}
\Crefname{table}{Tab.}{Tabs.}
\crefformat{section}{§#2#1#3}
\Crefformat{section}{§#2#1#3}
\crefalias{appsection}{appendix}
\crefalias{appsubsection}{appendix}
\crefalias{appsection}{Appendix}
\crefalias{appsubsection}{Appendix}
\AddToHook{cmd/appendix/before}{\crefalias{section}{appendix}}

\crefname{algorithm}{Alg.}{Algs.}
\Crefname{algorithm}{Alg.}{Algs.}
\creflabelformat{appsection}{#2\textup{App.~}#1#3}
\crefformat{appendix}{Appendix~#2#1#3}
\Crefformat{appendix}{Appendix~#2#1#3}

\setlist{leftmargin=*, itemsep=0pt, topsep=0pt}

\icmltitlerunning{Non-Parametric Optimization for Scalable Learning in Stochastic Decision Problems}

\begin{document}

\twocolumn[

\icmltitle{Non-Parametric Optimization for Scalable Learning\\in Stochastic Decision Problems}

\icmlsetsymbol{equal}{*}

\begin{icmlauthorlist}
	\icmlauthor{Mohsen Amidzadeh}{yy}
	\icmlauthor{Lauri Viitasaari}{zz}
	\icmlauthor{Mario Di Francesco}{yy}
\end{icmlauthorlist}

\icmlaffiliation{yy}{Department of Computer Science, Aalto University, Finland}
\icmlaffiliation{zz}{Department of Information and Service Management, Aalto University, Finland}

\icmlcorrespondingauthor{Mohsen Amidzadeh}{\texttt{mohsen.amidzade@aalto.fi}}
\icmlcorrespondingauthor{Lauri Viitasaari}{\texttt{lauri.viitasaari@aalto.fi}}

\icmlkeywords{Stochastic optimization, Non-parametric methods}

\vskip 0.3in
]

\printAffiliationsAndNotice{}

\begin{abstract}
	Stochastic optimization (SO) plays a central role in decision-making under uncertainty. 
	Among SO problems, time-varying stochastic optimization (TV-SO)
	is particularly important due to its applications in adaptive control and machine learning. 
	Non-parametric approaches have been proposed for time-varying 
	\emph{deterministic} optimization; however, they have not been developed for their stochastic counterparts. This work addresses that gap by developing a stochastic variational framework based on Malliavin calculus. 
	This framework yields non-parametric optimality conditions for SO problems with stochastic decisions and supports the design of a scalable deep-learning algorithm that is insensitive to the parameterization dimension. This algorithm, called the Stochastic Path Follower (SPF), is applied to two important problems under distribution drift, namely least-squares recovery and logistic regression. Experimental results show that the proposed approach outperforms state-of-the-art learning-based and gradient-based methods in both performance and scalability.

\end{abstract}

\section{Introduction}

Stochastic optimization (SO) studies decision-making under uncertainty and has wide applications, for instance, 
in mathematical finance~\citep{strini2019,buehler2019deep}, 
control theory~\citep{Jin2022}, machine learning~\citep{bottou2018,sun2020,
Levine2020}. 
There has been extensive research on finding the optimal solution of such problems when the decision variable is deterministic~\citep{duchi2016,Cao2021,Cutler2021}. Recent advances have also addressed SO problems with stochastic decisions through neural network (NN) parameterizations~\citep{tzen2019neural,Li2020SDE,oh2019kernel,Kidger2021GAN}.
However, optimality conditions for \emph{SO problems with stochastic decisions} would benefit greatly from being insensitive 
to the parameterization dimension, thereby improving scalability.

This work addresses this gap by developing \emph{non-parametric} necessary conditions for optimally solving SO problems with stochastic decisions. The resulting non-parametric methodology leads to an algorithm that is independent of the parameterization dimension and is therefore both efficient and scalable. In particular, we focus on time-varying stochastic optimization (TV-SO), a class of SO problems in which the expected objective and\,/\,or constraints vary over time.
TV-SO has applications in important fields including adaptive control~\citep{Fazlyab2018,Simonetto2020,Simonetto2024} 
and machine learning~\citep{Cutler2023,Maity2023PredictCorect}, especially in settings where the data distribution changes over time (distribution drift) and the underlying environment is non-stationary.
Our approach is complementary to the dynamic-programming principle commonly used for sequential decision-making problems modeled by stochastic differential equations (SDEs) or formulated as Markov decision processes (MDPs)~\citep{FlemingSoner2006,BouchardTouzi2011,Puterman2014,amidzadeh2025fbmoac}. From a different analytical angle, the proposed Malliavin-based framework studies time-varying stochastic objectives through non-parametric optimality conditions.

There is a rich theory for analyzing optimality conditions in time-varying %
\emph{deterministic} problems~\citep{Simonetto2020,komzsik2019}; however, their stochastic counterparts have received much less attention in the literature.
Hence, we leverage Malliavin calculus to develop a stochastic variational framework for deriving optimality conditions in SO problems. Malliavin calculus is essential in this context, as it enables the formulation of \emph{non-parametric} optimality conditions. The resulting theory connects deterministic and stochastic optimization by generalizing classical optimality conditions to the stochastic setting. Since these conditions do not yield a closed-form expression for the optimal stochastic solution, we design a deep-learning algorithm, called the Stochastic Path Follower (SPF), to approximate it. This algorithm is built on a non-parametric framework that makes the learning procedure largely insensitive to the parameterization dimension, thereby providing a scalable and low-complexity approach. %
Pathwise differentiation methods (PDMs), which estimate gradients of SDE-driven systems~\citep{tzen2019neural,liu2019neural}, can also be used to solve the SO problems of interest through a parameterized NN. However, their learning mechanism strongly depends on the parameterization dimension; thus, it does not scale to large models. %
Tab.~\ref{Table:complexity} compares our proposed SPF against PDM.

\begin{table}[t!]
	\centering
	\caption{Comparison of the SPF and PDM algorithms, 
        where $n$ denotes the dimension of the decision variable 
        and $\phi$ is the parameterization dimension.
        For SPF, both the required SDEs and the learning loss 
        do not depend on $\phi$, while they do for PDM.}
	\label{Table:complexity}
	\small
	\begin{tabular}{lll}
		\toprule
		\multirow{2}{*}{\textbf{Algorithm}} & \textbf{Dimension of}  		& \textbf{SDEs and learning loss}   \\[-2pt]
		& \textbf{required SDEs}      &  \textbf{explicitly depend on $\phi$}     \\
		\cmidrule(lr){1-1}\cmidrule(lr){2-3}
		SPF (ours) &  $n + n^2$  & \textcolor{darkgreen}{No}  \\
		PDM          &  $n + n$\textcolor{darkred}{$\phi$}    & \textcolor{darkred}{Yes}     \\
		\bottomrule
	\end{tabular}
	\small
	\vspace{-10 pt}
\end{table}

In summary, we establish the following contributions.

\begin{itemize}
	\item We study time-varying stochastic optimization\footnote{\cref{App:FSO} also addresses functional stochastic optimization as a different class of SO problems.} (TV-SO). %
	We then develop a non-parametric variational framework based on Malliavin calculus to derive optimality conditions for TV-SO problems.
	\item We accordingly develop an efficient NN algorithm to solve TV-SO. This algorithm computes an optimal solution in a manner that is largely insensitive to the parameterization dimension and is therefore both scalable and efficient.
	\item We compare our proposed solutions with state-of-the-art (SOTA) algorithms through experiments on two important use cases: least-squares recovery and logistic regression. %
	The experimental results show that our algorithm outperforms SOTA methods in terms of both performance and computational complexity.
\end{itemize}

\section{Related Work}
\label{sec:related-work}

\fakeparagraph{Time-Varying Optimization}
This work is broadly related to time-varying optimization problems. Such problems involve parametric programming, in which the objective function, the constraints, or both vary continuously over time~\citep{Simonetto2020}. Since optimality is sought over a given time horizon, the optimality conditions yield trajectories that can be represented by an ordinary differential equation (ODE), termed a path-differential equation~\citep{Fazlyab2018}. This ODE traces the optimal trajectory over time while accounting for the time-varying nature of the problem.

In contrast, this paper considers time-varying \emph{stochastic} optimization with \emph{stochastic decisions} as a novel research setting and develops optimality conditions that characterize the evolution of optimal stochastic paths over time. We employ a diffusion process to solve these optimality conditions, which leads to a stochastic path-differential equation. In this sense, the resulting optimality conditions constitute a stochastic generalization of those for time-varying optimization problems.

Our work is also related to the gradient-based algorithm developed in~\citep{Cutler2023}, termed proximal stochastic gradient (PSG). PSG addresses convex optimization problems under distribution drift and tracks the optimal solution using a stochastic algorithm with iterate averaging. In contrast, we directly target the optimality conditions and derive a structured search direction based on the proposed stochastic variational framework.

\fakeparagraph{Neural Stochastic Differential Equations}
Recent works estimate gradients of SDE-driven functions using NN parameterization with backpropagation or adjoint methods~\citep{Li2020SDE,Kidger2021GAN}. However, they do not address the joint optimization of a sequence of objectives over an entire time horizon, and they inherently adopt parametric methodologies: they optimize a finite set of network weights rather than the stochastic decision itself. In contrast, our framework performs joint optimization of SO problems over the whole horizon using a non-parametric and scalable methodology, without being restricted to a specific parameterization. %

\section{Background}
\label{sec:background}

This section provides the background needed on time-varying optimization
and key results from Malliavin calculus.

\subsection{Time-Varying Optimization}
\label{sec:tvo}

Consider the following time-varying optimization (TV-O) problem:
\begingroup\makeatletter\def\f@size{9.5}\check@mathfonts
\begin{align}\label{eq:TVO}
\mbox{TV-O:}\quad \min_{(\textbf{x}_t)_{t\in[0,T]}}~~f(\textbf{x}_t,t),\quad \forall t\in [0,T], 
\end{align}
\endgroup
where $\textbf{x}_0$ is given, $f:\mathbb{R}^n\times [0,T] \to \mathbb{R}$ is a twice-differentiable time-varying function,
and the goal is to track the sequence of solutions $(\textbf{x}_t)_{ t\in[0,T] }$ over time.
Since the problem must be optimized\footnote{Note that minimizing the aggregated objective $\int_0^T |f(\textbf{x}_t,t)|,dt$ does not, in general, yield the desired solution, since the minimizer of $f(\textbf{x}_t,t)$ may change sign at different time instances.} jointly for all $t\in [0,T]$, the optimality condition $\nabla_{\textbf{x}} f(\textbf{x}_t,t) = \textbf{0}$ must hold for every $t$. %
This leads to the following ordinary differential equation (ODE):
\begingroup\makeatletter\def\f@size{9.75}\check@mathfonts
\begin{align}\label{eq:dynamical-system}
	\frac{d }{dt} \nabla_\textbf{x}f(\textbf{x}_t,t) 
	= \nabla_{\textbf{x}}^2 f(\textbf{x}_t,t) \frac{d\textbf{x}_t}{dt}+ \frac{\partial}{\partial t}\nabla_{\textbf{x}} f(\textbf{x}_t,t) = \textbf{0}.
\end{align}
\endgroup
A path-following method then tracks the solution of TV-O by solving this ODE over $t\in[0,T]$~\citep{Kungurtsev2017,Amidzadeh2024}, using an ODE solver such as Euler's method. Thus, the optimal solution necessarily satisfies the following energy-functional loss:
\begin{align}\label{eq:non-parametric-loss}
	\int_0^T \left\| \nabla_{\textbf{x}}^2 f(\textbf{x}_t,t) \frac{d\textbf{x}_t}{dt}+ \frac{\partial}{\partial t}\nabla_{\textbf{x}} f(\textbf{x}_t,t)  \right\|^2 dt.
\end{align}

\subsection{Malliavin Calculus}
\label{sec:malliavin-derivative}

Malliavin calculus provides analytical tools to study the sensitivity of stochastic processes with respect to variations of a Wiener process $\textbf{W}=(w_s)_{s\in[0,t]}$. Consider a Wiener-driven process $x_t = \int_0^t g_\tau dw_\tau$ with a deterministic function $g_\tau$; its Malliavin derivative with respect to $w_s$, denoted by $D_s x_t$, is defined as $D_s x_t = g_s$ for all $s\leq t$. The chain rule also holds for this derivative, implying that for a Wiener-driven process $u_t$ and a smooth functional $F(\cdot)$ in the domain of the Malliavin derivative, $D_s F(u_t) = \frac{\partial F(u_t)}{\partial u_t} D_s u_t$.
Accordingly, we define the process $y_t = \int_0^t h_\tau dw_\tau$, where $h_\tau$ is Malliavin-differentiable, square-integrable, and adapted to the filtration generated by $w$. Its Malliavin derivative with respect to $w_s$ is then given by $D_s y_t = h_s + \int_0^t D_s h_\tau dw_\tau$. Throughout the rest of the paper, we consider processes with the same square-integrability and Malliavin-differentiability properties.

\fakeparagraph%
{Integration by Parts}
The integration-by-parts formula in Malliavin calculus~\citep{alos2021malliavin,nualart2018introduction} plays a central role in this work. It states that for any Malliavin-differentiable functional $F(\cdot)$ and any adapted, square-integrable process $u_s$, the following expression holds:
\begingroup\makeatletter\def\f@size{9.25}\check@mathfonts
$$
\mathbb{E}\left\{ F \int_0^t u_s dw_s  \right\} = \mathbb{E}\left\{ \int_0^t D_s F \:u_s ds \right\}.
$$
\endgroup

\fakeparagraph%
{Malliavin Derivative of SDEs}
We now consider the $n$-dimensional SDE: \vspace{-5 pt}
\begin{align}\label{Eq:nonparametricSDE}
	d\textbf{x}_t = \textbf{a}(\textbf{x}_t, t) dt + \sum_{l=1}^d \textbf{b}_l(\textbf{x}_t, t) dw^l_t
\end{align}
where $\textbf{a}: \mathbb{R}^n\!\times\![0,T] \!\to\! \mathbb{R}^n $ 
and $\textbf{B}=[\textbf{b}_1,...,\textbf{b}_d]:\mathbb{R}^n\!\times\![0,T] \!\to\! \mathbb{R}^{n\times d}$ 
are measurable functions with bounded partial derivatives.
Under these conditions, the Malliavin derivative of $\textbf{x}_t$ with respect to $w_s^i$ is given by~\citep{nualart2018introduction}:
\begingroup\makeatletter\def\f@size{8}\check@mathfonts
\begin{align}\label{Eq:MalliavinDerv0}
	\!\!\!\!	D_s^i \textbf{x}_t = \textbf{b}_i(\textbf{x}_t)\!
	+\!\sum_{l=1}^d \!\int_s^t \!\!\! \mathcal{J}_{\textbf{x}}\textbf{b}_l(\textbf{x}_\tau) D_s^i \textbf{x}_\tau d w_\tau^l \!
	+\!\int_s^t \!\!\! \mathcal{J}_{\textbf{x}}\textbf{a}(x_\tau) D_s^i \textbf{x}_\tau d \tau,
\end{align}
\endgroup
for $i\in\{1,...,d\}$, where $\mathcal{J}_{\textbf{x}} \textbf{a}$ (respectively, $\mathcal{J}_{\textbf{x}} \textbf{b}_l$) denotes the Jacobian matrix of the vector function $\textbf{a}$ (respectively, $\textbf{b}_l$) with respect to $\textbf{x}$.

In addition to~\cref{Eq:MalliavinDerv0}, we can express the Malliavin derivative of $\textbf{x}_t$ in terms of the stochastic flow $\frac{\partial \textbf{x}_t}{\partial \textbf{x}_0}$. To this end, define the Malliavin matrix $D_s\textbf{x}_t:=[D_s^1\textbf{x}_t,\ldots,D_s^d\textbf{x}_t]$. Then we have~\citep{nualart2018introduction}:
\begin{align}\label{Eq:MatrixMalliavin}
	D_s\textbf{x}_t = \boldsymbol{\Gamma}_{s,t} \: \textbf{B}(\textbf{x}_s, s),
\end{align}
where $\boldsymbol{\Gamma}_{s,t} := \textbf{Y}_t \textbf{Z}_s$ is the $s$-to-$t$ stochastic flow, with the flow matrix $\textbf{Y}_t = \frac{\partial \textbf{x}_t}{\partial \textbf{x}_0} \in \mathbb{R}^{n\times n}$ and its inverse $\textbf{Z}_t = \textbf{Y}_t^{-1}$ 
satisfying the following variational SDEs:
\begingroup\makeatletter\def\f@size{9.5}\check@mathfonts
\begin{align}\label{Eq:Y-Z}
	\textbf{Y}_t =& \textbf{I}+\sum_{l=1}^d \int_0^t \mathcal{J}_\textbf{x} \textbf{b}_l\left(\textbf{x}_s,s\right) \textbf{Y}_s d w_s^l
	+\int_0^t \mathcal{J}_\textbf{x} \textbf{a}\left(\textbf{x}_s,s\right) \textbf{Y}_s d s \notag\\
	\textbf{Z}_t =& \textbf{I}-\sum_{l=1}^d \int_0^t\! \textbf{Z}_s\bigg( \mathcal{J}_\textbf{x} \textbf{b}_{l}\left(\textbf{x}_s,s\right) d w_s^l\notag\\
	&\hspace{4pt}-\int_0^t \textbf{Z}_s\!\left(\mathcal{J}_\textbf{x} \textbf{a}\left(\textbf{x}_s,s\right)\right.
	\!-\!\sum_{l=1}^d\! \mathcal{J}_\textbf{x} \textbf{b}_{l}\left(\textbf{x}_s,s\right)^2 \bigg) d s.
\end{align}
\endgroup

\section{%
Time-Varying Stochastic Optimization}
\label{sec:problem-statement}

Throughout, we consider a filtered probability space $(\Omega, \mathcal{F}, \mathcal{F}_t, \mathbb{P})$,
where $\Omega$ is the sample space,
$\mathcal{F}$ is a $\sigma$-algebra of events on $\Omega$,
and $\mathbb{P} : \mathcal{F} \to [0,1]$ is a probability measure. %
On this space, we consider a $d$-dimensional Wiener process $\mathbf{W} = (\mathbf{w}_t)_{t \geq 0}$
\textcolor{black}{with diffusion matrix $\textbf{I}$}.
The filtration $(\mathcal{F}_t)_{t \geq 0}$ is taken to be the natural filtration of $\mathbf{W}$, namely $\mathcal{F}_t = \sigma((\textbf{w}_s)_{s\in[0,t]})$, which encodes all the information up to time $t$.

We next focus on time-varying stochastic optimization (TV-SO); \cref{App:FSO} studies another class, namely functional stochastic optimization. In TV-SO, the decision variable is a Wiener function and the objective is a Wiener functional. We therefore need stochastic variational tools based on Malliavin calculus to analyze the optimality conditions of such problems. This, in turn, enables us to characterize the optimal decision by perturbing the solution along Wiener-driven paths.

Specifically, we consider a time-varying stochastic optimization (TV-SO) problem of the form
\begin{align}\label{Eq:TV-SO}
	\mbox{TV-SO}: \qquad	&\min_{ \textbf{X} }  \!\!
	\underset{
		\begin{tiny} 
			\begin{array}{c}
				(\boldsymbol{\Xi},\textbf{X})\!\sim\!\mathbb{P}\!\otimes\!\mathbb{Q}
			\end{array} 
		\end{tiny} 
	}{\mathbb{E}}\hspace{-5 pt}
	\big\{ f(\textbf{x}_t, \boldsymbol{\xi}_t, t) \big\},
\end{align}
where the optimization is carried out jointly for every $t\in[0,T]$. Note that minimizing the aggregated objective $\int_0^t |\mathbb{E}f(\textbf{x}_t,\boldsymbol{\xi}_t,t)|dt$ does not, in general, yield the desired solution, since the minimizer of $\mathbb{E}f(\textbf{x}_t,\boldsymbol{\xi}_t,t)$ may change sign at different time instances.

We now make the following assumptions:

\noindent \textbf{Assumption I:} $\boldsymbol{\Xi}=(\boldsymbol{\xi}_t \in \mathbb{R}^m)_{t\in[0,T]}$ is a known Malliavin-differentiable process, adapted to the filtration of $\textbf{W}$ under the law $\mathbb{P}$.

\noindent \textbf{Assumption II:} $f:\mathbb{R}^n\times\mathbb{R}^m\times[0,T] \to \mathbb{R}$ is square-integrable and twice continuously differentiable with respect to $\textbf{x}_t$ and $\boldsymbol{\xi}_t$, and represents a time-varying performance measure.

\noindent \textbf{Assumption III:} In contrast to SO problems with a deterministic decision, the decision here is assumed to be a \emph{Malliavin-differentiable} process $\textbf{X}=(\textbf{x}_t \in \mathbb{R}^n)_{t\in[0,T]}$ under law $\mathbb{Q}$.

Note that, for optimization purposes, $\mathbb{Q}$ should be induced by the same Wiener process from which $\boldsymbol{\Xi}$ is derived. Consequently, $\textbf{x}_t$ is viewed as a \emph{Wiener-driven function}. Moreover, Malliavin-differentiable functions form a broad class of processes (including deterministic ones) and do not impose restrictive assumptions on Wiener processes in practice.

\fakeparagraph{Illustrative Example}
As an illustrative example, consider the SO problem 
$ \min_{x(\cdot)}~ \mathbb{E}_{\eta \sim \mathcal{N}(0,1)}\big\{ x(\eta)^4 \: \eta^2 + 4x(\eta) \: \eta \big\}$. If $x$ is constrained to be deterministic, the optimal solution is $x^* = 0$ with objective value $0$. Allowing $x$ to take the linear form $x(\eta)=a+b\eta$ yields the optimal solution $ x^*(\eta) = -\eta/\sqrt[3]{5} $ with objective value $-1.216$, which is smaller than in the deterministic case. This example also shows that the condition $\mathbb{E}\{ 4x(\eta)^3 \eta^2 + 4 \eta \} = 0$ cannot uniquely determine the optimal solution, since it admits infinitely many solutions $\{ a/\eta^{1/3}, a\eta, a\eta^{5/3},... \}$ for a constant $a$.

\begin{table}[b]
	\centering
	\small\setlength{\tabcolsep}{5pt}
	\caption{Comparison of optimality conditions of TV-O~(\ref{eq:TVO}) and TV-SO~(\ref{Eq:TV-SO}) problems.}
	\begin{tabular}{c l}
		\toprule
		Problem & \hspace{43 pt} Condition \\
		\cmidrule{1-1}\cmidrule{2-2}
		TV-O  	& \hspace{13 pt}$\nabla_\textbf{x}^2f \: \textcolor{darkblue}{\frac{d\textbf{x}_t}{dt}} 
		\hspace{7pt} + \textcolor{darkgreen}{\frac{\partial}{\partial t}} \hspace{3pt} \nabla_\textbf{x}f \hspace{23 pt} = \textbf{0}$ \\
		TV-SO 	& $\mathbb{E}\left\{ \nabla_\textbf{x}^2f \: \textcolor{darkblue}{D_s \textbf{x}_t}   
		+ \textcolor{darkgreen}{\nabla_{\boldsymbol{\xi}}}\nabla_\textbf{x}f \: \textcolor{darkred}{D_s\boldsymbol{\xi}} \right\}= \textbf{0}$ \\
		\bottomrule
	\end{tabular}
	\label{Eq:optimality-comparison}
\end{table}

As the law of $\textbf{X}$ is induced by the Wiener process $\textbf{W}$, the objective of $\mbox{TV-SO}$ can be reformulated as the following (Riemann) integral: 
\begin{align}
	\label{Eq:Riemann-integral}
	\int f(\textbf{x}_t(\textbf{W}),\boldsymbol{\xi}_t(\textbf{W}), t) \mathbb{P}(\textbf{W}) d\textbf{W}. 
\end{align}
Note that $\textbf{x}_t$ should be viewed as a \emph{Wiener function} rather than a simple variable, and the objective as a \emph{Wiener functional}. 
Hence, the expression in~\cref{Eq:Riemann-integral} does not explicitly 
depend on $\textbf{x}_t$, 
so 
$\nabla_\textbf{x} \int f(\textbf{x}_t(\textbf{W}),\xi_t(\textbf{W}),t) \mathbb{P}(\textbf{W})d\textbf{W}$ cannot be used
to obtain the optimal solution.
\textcolor{black}{In addition, as the illustrative example above shows, $\int \nabla_\textbf{x}f(\textbf{x}_t(\textbf{W}),\xi_t(\textbf{W}),t) \mathbb{P}(\textbf{W})d\textbf{W}$ 
does not necessarily lead to an optimum.}
This calls for stochastic variational analysis to characterize the optimality conditions. Moreover, we aim to extend the solution mechanism of the TV-O problem in~\cref{sec:tvo} to its stochastic counterpart in~\cref{Eq:TV-SO}. Accordingly, we develop a stochastic perturbation framework based on Malliavin calculus, which provides \emph{non-parametric} optimality conditions for optimizing $\mbox{TV-SO}$ at a given $t$.

The following theorem provides the necessary optimality conditions for TV-SO~(\ref{Eq:TV-SO}).
\begin{tcolorbox}[graybox]
	\begin{theorem}[Necessary optimality conditions]%
		\label{Thm:1}
		Consider problem $\mbox{TV-SO}$ under Assumptions I–III, with $\textbf{X}=(\textbf{x}_t)_{t\in[0,T]}$ and $\boldsymbol{\Xi}=(\boldsymbol{\xi}_t)_{t\in[0,T]}$ being $\mathcal{F}$-adapted, Malliavin-differentiable, and square-integrable processes. Then $\mbox{TV-SO}$ is optimized at a given instant $t$ if the following %
		necessary conditions hold: 
		\begin{align}\label{Eq:Conditions_TVSO}
			&\mathbb{E}\big\{ \nabla_\textbf{x}^2 f(\textbf{x}_t, \boldsymbol{\xi}_t, t) \: D_s^i \textbf{x}_t 
			+   \nabla_{\textbf{x}\boldsymbol{\xi}}f(\textbf{x}_t, \boldsymbol{\xi}_t, t)\: D_s^i \boldsymbol{\xi}_t \big\} = \textbf{0}, \notag \\
			&\mathbb{E}\big\{ \nabla_{\textbf{x}}f(\textbf{x}_t, \boldsymbol{\xi}_t, t) \big\} = \textbf{0},
		\end{align}
		for $s\!\in\![0,t]$, $i \in\{1,\ldots,d\}$, where $D_s^i \textbf{x}_t \!\in\! \mathbb{R}^n$ (respectively, $D_s^i \boldsymbol{\xi}_t \!\in\! \mathbb{R}^m$) is the Malliavin derivative of $\textbf{x}_t$ (respectively, $\boldsymbol{\xi}_t$) with respect to $w_t^i$ (the $i$-th component of $\textbf{w}_t$).
	\end{theorem}
\end{tcolorbox}
\begin{proof}
	See~\cref{App:proof1}.
\end{proof}

\begin{remark}
	Note that~\cref{Eq:Conditions_TVSO} introduces an additional condition required to ensure the optimality of the stochastic paths generated by $\textbf{X}$. This Malliavin-based term can also be written as 
	$$
	\mathbb{E}\big\{ D_s^i \nabla_x f(\textbf{x}_t,\boldsymbol{\xi}_t,t) \big\} = \textbf{0},\qquad i \in \{1,...,d\},
	$$
	which is the direct stochastic analogue of the deterministic stationarity condition $\frac{d }{dt} \nabla\textbf{x}f(\textbf{x}_t,t) = \textbf{0}$ from~\cref{eq:dynamical-system}, with the time derivative replaced by the Malliavin derivative. This analogy can also be seen by moving from~\cref{eq:dynamical-system} to~\cref{Eq:Conditions_TVSO} via %
	$\frac{d\textbf{x}_t}{dt} \to D_s\textbf{x}_t$ and $\frac{\partial}{\partial t}\nabla_{\textbf{x}}f \to \nabla_{\textbf{x}\boldsymbol{\xi}}f$ (see \cref{Eq:optimality-comparison}).
\end{remark}

As in Problem TV-O, we aim to optimize TV-SO jointly for every $t \in [0,T]$. This leads to a sequence of optimization problems over the entire horizon and requires the optimality conditions in~\cref{Eq:Conditions_TVSO} to hold for all $t \in [0,T]$. Accordingly, the stochastic trajectories should satisfy~\cref{Eq:Conditions_TVSO} at each time $t$, which yields a framework of \emph{stochastic path following} as the stochastic counterpart of the path-following method described in~\cref{sec:tvo}.

\begin{remark}
	A decision process $\textbf{X} = (\textbf{x}_t)_{t\in[0,T]}$ can be jointly optimized in $\mbox{TV-SO}$ for every $t\in[0,T]$ by enforcing the following energy-functional criterion:
	\begingroup\makeatletter\def\f@size{8.5}\check@mathfonts
	\begin{align}\label{Eq:crit2}
		&\hspace{4em}\mathcal{L}_{\rm TV}(\textbf{x},D\textbf{x})  = 
		\int_0^T \left(\Big\|\mathbb{E}\big\{ \nabla_{\textbf{x}}f(\textbf{x}_t, \boldsymbol{\xi}_t, t) \big\} \Big\|^2  \right.
		\\
		&\left.+
		\sum_{i=1}^d \int_0^t \! \Big\| \mathbb{E}\big\{ \nabla_\textbf{x}^2 f(\textbf{x}_t, \boldsymbol{\xi}_t, t) \: D_s^i \textbf{x}_t
		\!+\!   \nabla_{\textbf{x}\boldsymbol{\xi}}f(\textbf{x}_t, \boldsymbol{\xi}_t, t)\: D_s^i \boldsymbol{\xi}_t \big\} \Big\|^2 \!\! ds \! \right)\! dt \notag
	\end{align}
	\endgroup
	Notice that this criterion is the stochastic counterpart of the deterministic case in~\cref{eq:non-parametric-loss}.
\end{remark}

\section{Diffusion Process as a Candidate Decision}

Recall from~\cref{eq:dynamical-system} that the solution of a time-varying optimization problem can be obtained via an ODE. Accordingly, we consider an SDE to solve Problem TV-SO in~\cref{Eq:TV-SO}. For this purpose, we introduce the following diffusion process based on~\cref{Eq:nonparametricSDE}: %
\begingroup\makeatletter\def\f@size{8.5}\check@mathfonts
\begin{align}\label{Eq:diffusion}
	{\rm Diff}^{\rm \textbf{a}, \textbf{B}} \!=\!
	\Big\{
	\textbf{X}=\!(\textbf{x}_s)_{s\in[0,T]}\!:~
	d\textbf{x}_t = \textbf{a}(\textbf{x}_t, t) dt + \sum_{l=1}^d \textbf{b}_l(\textbf{x}_t, t) dw^l_t
	\Big\},
\end{align}
\endgroup
with given $\textbf{x}_0$. 
We impose standard bounded-derivative conditions on $\textbf{a}$ and $\{\textbf{b}_l\}_{l=1}^d$ to guarantee the existence of a unique solution to the process~\citep{Li2020SDE}. Hence, finding the probability law $\mathbb{Q}$ in Problem $\mbox{TV-SO}$ [i.e.,~\cref{Eq:TV-SO}] is equivalent to finding a diffusion process with optimal drift $\textbf{a}$ and diffusion functions $\textbf{B}=[\textbf{b}_1,\ldots,\textbf{b}_d]$. Moreover, the Malliavin derivatives $D_s\textbf{x}_t^i$ satisfy the variational SDE in~\cref{Eq:MalliavinDerv0}. Therefore, evaluating the optimality conditions in~\cref{Eq:crit2} depends mainly on the Malliavin derivative of the chosen diffusion process, which motivates us to develop a method for computing $D_s\textbf{x}_t$.

\subsection{\scalebox{0.95}{Computation of Malliavin Derivatives}}

We next explain how to compute $D_s\textbf{x}_t$. For this purpose, we use the following proposition.

\begin{tcolorbox}[graybox]
\begin{proposition}
    \label{Propos:BackwardSDE}
    Consider the Malliavin derivative $D_r \textbf{x}_t=[D_r^1 \textbf{x}_t,...,D_r^d \textbf{x}_t]$ in~\cref{Eq:MalliavinDerv0} for the SDE in~\cref{Eq:diffusion}. Then
    \begin{align}\label{Eq:vector_malliavin}
        D_r \textbf{x}_t = \boldsymbol{\Gamma}_{r,t} \:\textbf{B}(\textbf{x}_r,r),\qquad\text{for }r\leq t,
    \end{align}
    where $r$-to-$t$ stochastic flow $\boldsymbol{\Gamma}_{r,t} \in \mathbb{R}^{n\times n}$ 
    represents a linear SDE, adapted to the filtration $\mathcal{F}_{r}$, given by
    \begingroup\makeatletter\def\f@size{9.25}\check@mathfonts
    \begin{align}\label{Eq:BSDE_gamma}
        &%
        d \boldsymbol{\Gamma}_{r,t} = \Bigg( \sum_{l=1}^d \mathcal{J}_{\textbf{x}}\textbf{b}_l(\textbf{x}_t,t)\: dw_t^l 
        + \mathcal{J}_{\textbf{x}}\textbf{a}(\textbf{x}_t,t) dt  \Bigg)\boldsymbol{\Gamma}_{r,t},
    \end{align}
    \endgroup
    with initial condition $\boldsymbol{\Gamma}_{r,r}=\textbf{I}$.
\end{proposition}
\end{tcolorbox}
\begin{proof}
    See~\cref{App2}.
\end{proof}

Combined with an SDE simulator, Proposition~\ref{Propos:BackwardSDE} enables the recursive computation of the Malliavin derivative $D_s\textbf{x}_t$. 
Starting from $\boldsymbol{\Gamma}_{r,r}= \textbf{I}$, we can increment $t$ to evaluate the right-hand side of~\cref{Eq:BSDE_gamma} 
for $t \in [0,T]$. 
In particular, we can use the Euler--Maruyama scheme, which has a convergence rate of \({O}(\sqrt{\Delta r})\)~\citep{kloeden1992numerical},
to simulate $\boldsymbol{\Gamma}_{r,t}$ and $D_r\textbf{x}_t$ for $r\in(0,t]$ given $t$.
\Cref{Alg:backward-malliavin} provides the corresponding pseudo-code based on the Euler--Maruyama scheme.

\begin{algorithm}[t]
	\small
	\caption{Computation of the Malliavin derivative via the Euler--Maruyama scheme}
	\label{Alg:backward-malliavin}
	\DontPrintSemicolon
	\KwData{Forward trajectory $\{\textbf{x}_r, \textbf{w}_r\}_{r=0:\Delta r:t}$, time‐step $\Delta r$.}
	\textcolor{darkblue}{\tcp{\small Initialize via terminal condition:}}
	$\boldsymbol{\Gamma}_{r,r} \gets \textbf{I}$\;
	\textcolor{darkblue}{\tcp{ Solve the SDE in~\cref{Eq:BSDE_gamma}:}}
	\For{$t \gets 0$ \KwTo r with step $\Delta t$}{
		Compute\vspace{-6pt} %
		\scalebox{0.9}{$
			\displaystyle \Delta \boldsymbol{\Gamma}_{r,t}
			\;=\;\,\Bigl( \mathcal{J}_{\textbf{x}} \textbf{a}(\textbf{x}_r,r) \Delta r
			+\sum_{l=1}^d \mathcal{J}_{\textbf{x}} \textbf{b}_l(\textbf{x}_r,r)\bigl(\textbf{w}_{r}-\textbf{w}_{r-\Delta r}\bigr)\Bigr) \boldsymbol{\Gamma}_{r,t}
			$}\vspace{-10pt}\;
		Update 
		$\boldsymbol{\Gamma}_{\,r\,,\,t+\Delta t}
		\;\gets\;\boldsymbol{\Gamma}_{r,t} + \Delta \boldsymbol{\Gamma}_{r,t}$\;
	}
	\textcolor{darkblue}{\tcp{ Recover Malliavin derivatives:}}
	\For{$r\in\{0,\Delta r,2\Delta r,\dots,t\}$}
	{
		$D_r \textbf{x}_t \;\gets\;\boldsymbol{\Gamma}_{r,t}\;\textbf{B}\bigl(\textbf{x}_r,r\bigr)$\;
	}
	\textbf{return} $(D_r\textbf{x}_t)_{r=0:\Delta r:t}$
	\small
\end{algorithm}

\section%
{The Stochastic Path Follower Algorithm}
This section presents an adaptable deep-learning algorithm for TV-SO problems based on the variational framework of~\cref{Thm:1}. As the algorithm follows stochastic optimal paths for $\mbox{TV-SO}$, we refer to it as the Stochastic Path Follower (SPF). To this end, we design diffusion processes that minimize the criterion in~\cref{Eq:crit2}. Although the framework developed in~\cref{sec:problem-statement} is independent of any specific parameterization technique, we adopt NN parameterization due to its demonstrated effectiveness in learning SDEs~\citep{Kidger2021GAN,Li2020SDE}. Specifically, we consider $\boldsymbol{\theta}$-parametric NNs for the drift and diffusion functions in~\cref{Eq:diffusion}. In detail, we set
\begingroup\makeatletter\def\f@size{8.5}\check@mathfonts
\begin{align}\label{Eq:parametricSDE}
    \vspace{-5 pt}
    {\rm SDE}( \textbf{a},\textbf{B},\textbf{x}_0,\boldsymbol{\theta}):~ d\textbf{x}_t = \textbf{a}(\textbf{x}_t, t, \boldsymbol{\theta}) dt 
    + \sum_{l=1}^d\textbf{b}_l(\textbf{x}_t, t, \boldsymbol{\theta}) d\textbf{w}_t^l,
\end{align}
\endgroup
with given $\textbf{x}_0$, where $\textbf{a}: \mathbb{R}^n\times[0,T]\times\Theta \to \mathbb{R}^n $ and $\textbf{B}=[\textbf{b}_1,\ldots,\textbf{b}_d]:\mathbb{R}^n\times[0,T]\times\Theta \to \mathbb{R}^{n\times d}$ are two $\boldsymbol{\theta}$-parametric NNs, and $\Theta \subset \mathbb{R}^{\phi}$ is the corresponding parameter set. We then apply the SPF algorithm to solve $\mbox{TV-SO}$ problems.

SPF consists of the following three main phases. %

\fakeparagraph{Forward pass}
Simulate the SDE in~\cref{Eq:parametricSDE} to obtain the stochastic paths of $(\textbf{x}_t)_{t\in[0,T]}$. In addition, compute the Malliavin derivatives $D_s\textbf{x}_t$ using Proposition~\ref{Propos:BackwardSDE}.

\fakeparagraph{Loss evaluation}
Compute the expected loss associated with Problems $\mbox{TV-SO}$, as given in~\cref{Eq:crit2}.

\fakeparagraph{Parameter update}
Update the parameters of the drift and diffusion functions using an Adam-type optimizer.

Pseudocode %
for the SPF algorithm is provided in~\cref{Alg:stochastic-path-following}, and to guarantee the convergence of the SPF algorithm, we rely on the results in~\citep{Adamconvergence2022}. In particular, we follow Algorithm 2 and the setup of Theorem 17 in that work. Moreover, to achieve an $\epsilon$-stationary point, we choose the batch size as $p = 1/(\epsilon \: \phi)$, where $\phi$ denotes the NN parameterization dimension (see~\cref{Sec:Eval} for the specific values used in our experiments). %

\begin{algorithm}[t]
	\small
	\caption{{Stochastic Path Follower (SPF)}}
	\label{Alg:stochastic-path-following}
	\textbf{Goal:} Solving stochastic optimization problem $\mbox{TV-SO}$ .\\ %
	\KwIn{
		Neural drift and diffusion functions $\mathbf{a}(\cdot, t, \boldsymbol{\theta}): \mathbb{R}^{n \times p} \times [0,T] \times \Theta \rightarrow \mathbb{R}^{n \times p}$, 
		$\mathbf{B}(\cdot, t, \boldsymbol{\theta}): \mathbb{R}^{n \times p} \times [0,T] \times \Theta \rightarrow \mathbb{R}^{n \times p \times d}$,
		initial process $\mathbf{x}_0 \in \mathbb{R}^{n \times p}$,
		number of sample paths $P$,
		number of episodes $N_{\text{eps}}$.
	}
	\KwOut{\scalebox{1.0}[1.0]{Diffusion process $\textbf{X}$ with optimal solution $\boldsymbol{\theta}$ of $\mbox{TV-SO}$}.\\}
	
	\For{$\text{E} = 1$ \KwTo $N_{\text{eps}}$}{
		\textcolor{darkblue}{\tcp{ Simulate the neural  SDE to obtain sample paths:}}
		$(\mathbf{x}_t, \mathbf{w}_t) \leftarrow \texttt{SDEsolve}(\mathbf{a}, \mathbf{B}, \mathbf{x}_0, \boldsymbol{\theta})$\;
		
		\textcolor{darkblue}{\tcp{ Compute Malliavin derivative $D_s\mathbf{x}_t$ by simulating the SDE of the s-to-t stochastic flow $\boldsymbol{\Gamma}_{s,t}$}}
		$D_s \mathbf{x}_t \leftarrow \texttt{SDEsolve}(\mathcal{J}_{\textbf{x}}\mathbf{a}, \mathcal{J}_{\textbf{x}}\mathbf{B}, \boldsymbol{\Gamma}_{s,s})$\;
		
		\textcolor{darkblue}{\tcp{ Evaluate the energy-functional loss $\mathcal{L}(\theta;\textbf{x},D\textbf{x})$}}
		$\mathcal{L}(\boldsymbol{\theta}) \leftarrow \mathcal{L}(\theta;\textbf{x},D\textbf{x})$\;
		
		\textcolor{darkblue}{\tcp{ Update parameters of neural SDE:}}
		$\boldsymbol{\theta} \leftarrow \text{AdamOptimizer}(\boldsymbol{\theta}, \mathcal{L}(\boldsymbol{\theta}))$\;
	}
	\small
\end{algorithm}

\section{Performance Evaluation}
\label{Sec:Eval}

This section evaluates the proposed stochastic variational framework through illustrative experiments and selected applications, namely least-squares recovery and logistic regression. 
\textcolor{black}{The source code of our solution is available at \url{https://github.com/amidimohsen/StochasticPathFollower}.}

We now introduce several algorithms from the literature as state-of-the-art (SOTA) baselines for comparison. As a \emph{learning-based} approach, we adopt the pathwise differentiation method (PDM)~\citep{tzen2019neural,liu2019neural}. PDM can be used for \emph{joint} optimization over the entire time horizon (i.e., optimization for every $t\in [0,T]$), in contrast to adjoint sensitivity approaches~\citep{Li2020SDE}. It yields \emph{parametric} optimality conditions based on a prescribed diffusion process (see~\cref{App:PDM} for more details).

In addition, we consider the proximal stochastic gradient (PSG) method from~\citep{Cutler2023} as a \emph{gradient-based} baseline. This algorithm addresses convex optimization problems under distribution drift and is a stochastic gradient descent method. PSG tracks the optimal decision process using a stochastic algorithm based on iterate averaging.

\subsection{Benchmark: Stochastic Extrapolation}

Here, we consider a few simple yet representative problems to \emph{benchmark} the SPF algorithm against PDM (note that PDM is a baseline approach for optimizing SDE-driven functions and can also be applied to time-varying optimization over an entire time horizon). Specifically, we aim to find optimal stochastic trajectories that optimize expected time-varying objectives. As a practical methodology, we extrapolate known optimal solutions of these objectives at time $t=0$ into the unknown paths at $t>0$ using the optimality condition in~\cref{Thm:1}. We therefore consider objectives of the form $\mathbb{E}f(\textbf{x}_t,\textbf{w}_t,t)$, where $\textbf{x}_t$ is the decision process, $\textbf{w}_t$ is the Wiener process, and $f(\cdot,\cdot,t)$ is a time-varying function. Moreover, we assume that the solution is known at $t=0$. The goal is then to extrapolate this solution from $t=0$ to any $t\in(0,T]$ by following the stochastic paths that optimize the objective of interest. In what follows, we first revisit the illustrative example presented in~\cref{sec:problem-statement}, and then evaluate a constrained TV-SO problem (see~\cref{App:toy-setup} for details of the experimental settings).

\fakeparagraph%
{Scalar TV-SO (Exp1)}
\label{sec:exp1}
\begin{figure*}[t]
\centering
\subcaptionbox{\label{Fig:Ex1}}{%
	\hspace{-4 pt}%
	\includegraphics[width=0.2425\textwidth]{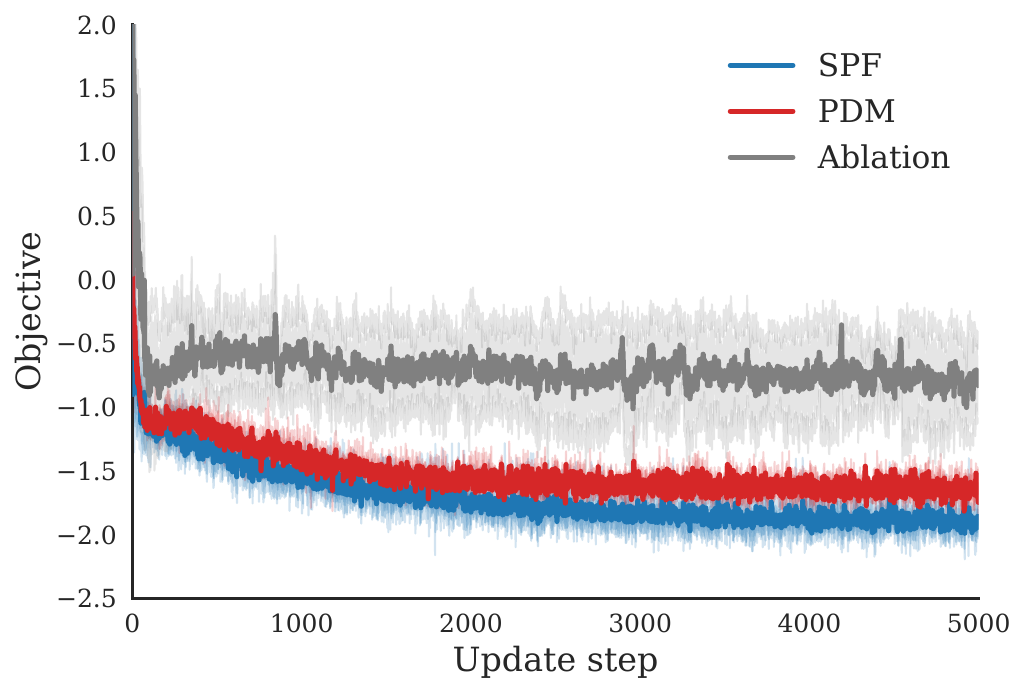}
}
\subcaptionbox{\label{Fig:optimal-values}}{%
	\hspace{-4 pt}%
	\includegraphics[width=0.2425\textwidth]{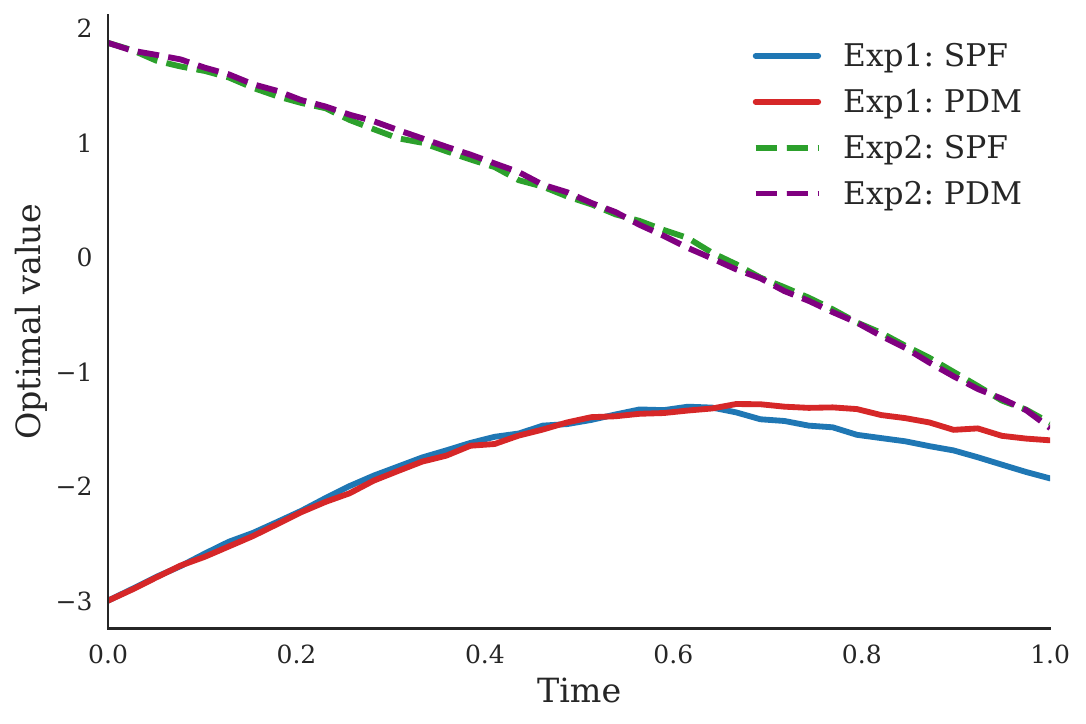}
}
\subcaptionbox{\label{Fig:Ex2-sol1}}{%
	\hspace{-4 pt}%
	\includegraphics[width=0.2425\textwidth]{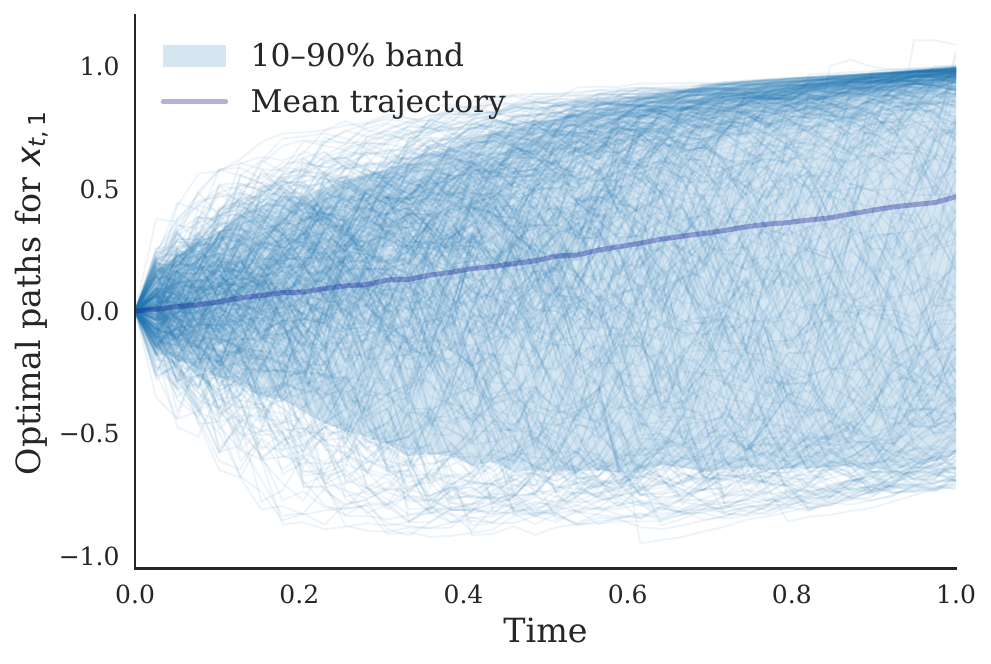}
}
\subcaptionbox{\label{Fig:Ex2-sol2}}{%
	\hspace{-4 pt}%
	\includegraphics[width=0.2425\textwidth]{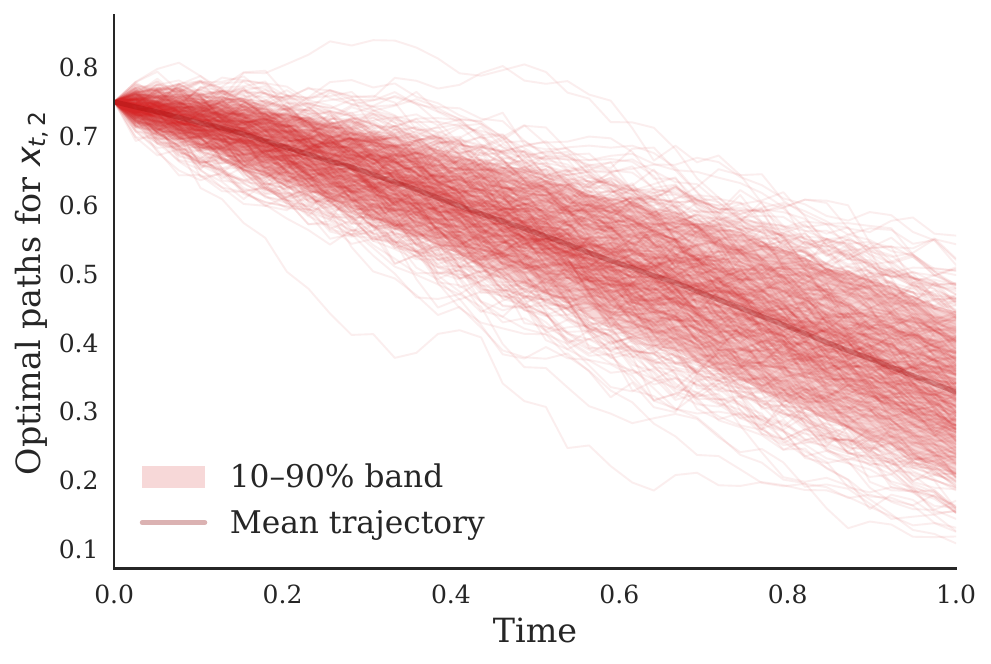}
}
\caption{
	\subref{Fig:Ex1} Learning performance and ablation results for problem Exp1 at $t=1$ for the SPF and PDM methods; the ablation curves are obtained by omitting the Malliavin-based term in~\cref{Eq:crit2}.
	\subref{Fig:optimal-values} Optimal values for each $t\in[0,1]$ in problems Exp1 and Exp2.
	\emph{Non-degenerate} optimal stochastic trajectories obtained by SPF in problem Exp2 for all $t \in [0,1]$ for \subref{Fig:Ex2-sol1} $x_{t,1}$ and \subref{Fig:Ex2-sol2} $x_{t,2}$.
}
\label{Fig:Ex1_val_sols}
\end{figure*}
\begin{table}[t!]
	\centering
	\caption{Average runtime per update step (in seconds).}\label{Table:runtime}
	\small\setlength{\tabcolsep}{5.25pt}
	\begin{tabular}{lrr}
		\toprule
		\textbf{Problem} & \multicolumn{1}{c}{\scalebox{.95}[1.0]{\textbf{SPF (ours)}}} & \multicolumn{1}{c}{\textbf{PDM}} \\
		\cmidrule(lr){1-2} \cmidrule(lr){3-3} 
		Scalar TV-SO (Exp1)
		& \textcolor{darkgreen}{$< 1$} & \textcolor{darkred}{9} \\
		{Constrained TV-SO (Exp2)
		}
		&\textcolor{darkgreen}{$< 1$}  & \textcolor{darkred}{41}  \\
		\bottomrule
	\end{tabular}
	\small
	\vspace{-10 pt}
\end{table}
Consider the following problem:
\begin{align}%
\!\!\!\!\!\!	\min_{{\rm Diff}^{\textbf{a},\textbf{B}}}
\underset{(w,x)\sim\mathbb{P}\otimes\mathbb{Q}}{\mathbb{E}}~x_t^4\:(1-t+ t w_t^2) + 4\:x_t\:(1-t+  t w_t) ,\nonumber
\end{align}
for every $t\in[0,1]$. 
The objective at $t=1$ becomes $\mathbb{E}_{w_1 \sim \mathcal{N}(0,1)}\{ x^4w_1^2 + 4xw_1 \}$, since $w_1 \sim \mathcal{N}(0,1)$, which matches the illustrative example in~\cref{sec:problem-statement}. For $t=0$, the optimal solution is $x_0^* = -1$. We then extrapolate this deterministic optimal solution at $t=0$ into a stochastic path for $t\in (0,1]$ using the stochastic path-following method implemented by SPF. To this end, we consider a diffusion process initialized at $x_0^*$. For comparison, we also apply the parametric PDM algorithm. %

\Cref{Fig:Ex1} reports the learned optimal values at $t=1$ obtained by the SPF and PDM algorithms. The importance of the proposed framework has been established theoretically in~\cref{sec:problem-statement}. Here, we further assess it through an ablation study of SPF. Specifically, \cref{Fig:Ex1} (curve with legend \emph{Ablation}) shows the learning performance when the Malliavin-based term in the optimality criterion in~\cref{Eq:crit2}, namely \scalebox{0.9}{$\mathbb{E}\big\{ \nabla_\textbf{x}^2 f(\textbf{x}_t, \boldsymbol{\xi}_t, t) \: D_s^i \textbf{x}_t 
\!+\!   \nabla_{\textbf{x}\boldsymbol{\xi}}f(\textbf{x}_t, \boldsymbol{\xi}_t, t)\: D_s^i \boldsymbol{\xi}_t \big\}$}, is omitted. In this case, the optimal solution cannot be recovered. \Cref{Fig:optimal-values} (curves with legend \emph{Exp1}) shows the resulting optimal values for all $t\in[0,1]$. SPF not only performs competitively with PDM, but is also computationally more efficient than the conventional PDM approach, as indicated by the average runtime per update step reported in~\cref{Table:runtime}.

\fakeparagraph%
{Constrained TV-SO (Exp2)}
\label{sec:exp2}
We now consider:
\begin{align}%
\!\!\!\!	\min_{{\rm Diff}^{\textbf{a},\textbf{B}}}  \underset{(\textbf{w},\textbf{x})\sim\mathcal{P}\otimes\mathbb{Q}}{\mathbb{E}}
\big\{ \textbf{f}(\textbf{x}_t)^\top \textbf{g}(\textbf{w}_t,t)  \big\}, ~~\mbox{s.t.}~~ \mathbb{E} \{\textbf{1}^\top \textbf{x}_t\}=3/4,\nonumber
\end{align}
optimized $\forall t \in [0,1]$, where \scalebox{0.9}{$\textbf{x}_t=[x_{t,1}, x_{t,2}]$, $\textbf{w}=[w^1,w^2]^\top$} is a vector Wiener process, \scalebox{0.9}{$\textbf{f}(\textbf{x}_t) = [x_{t,1}^4, 4x_{t,1}, 2x_{t,2}^2, x_{t,2}]^\top$,} and \scalebox{0.9}{$\textbf{g}(\textbf{w}_t,t) = [g_0(w_t^1,t), g_0({w_t^1}^2,t), g_0(w_t^2,t), g_0({w_t^2}^2,t)]^\top$} with \scalebox{0.9}{$g_0(w,t) = 1-t+ t w$}.

The problem has optimal solution $(x_{0,1}^*,x_{0,2}^*) = (0,3/4)$ at $t=0$. The goal is to determine the optimal paths for $t>0$ using SPF. We therefore employ a vector SDE initialized at $(x_{0,1}^*,x_{0,2}^*)$ and then apply~\cref{Alg:stochastic-path-following}.

\Cref{Fig:optimal-values} (the curves with legend \emph{Exp2}) shows the obtained optimal values for all $t\in(0,1]$. The SPF solution is consistent with that of PDM, but it exhibits significantly lower computational complexity (see~\cref{Table:runtime}). \Cref{Fig:Ex2-sol1,Fig:Ex2-sol2} illustrate the optimal trajectories of the resulting diffusion process optimized by SPF for $t\in[0,1]$, demonstrating that SPF maintains rich and non-degenerate dynamical behavior.

\subsection{Least-Squares Recovery}
\label{Sec:LS}

\begin{figure*}[t]
	\centering
	\subcaptionbox{\label{subfig:LS-1}}
	{
		\hspace{-15pt}
		\includegraphics[width=0.26\textwidth]{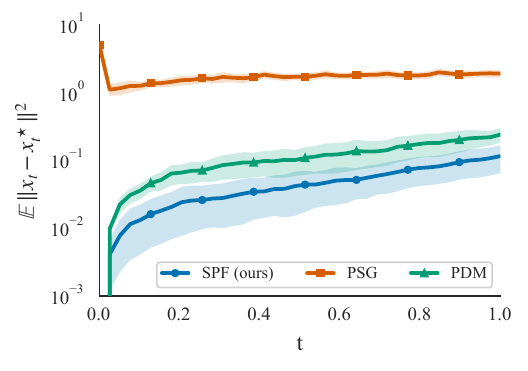}
	}
	\subcaptionbox{\label{subfig:LS-2}}
	{
		\hspace{-15pt}
		\includegraphics[width=0.26\textwidth]{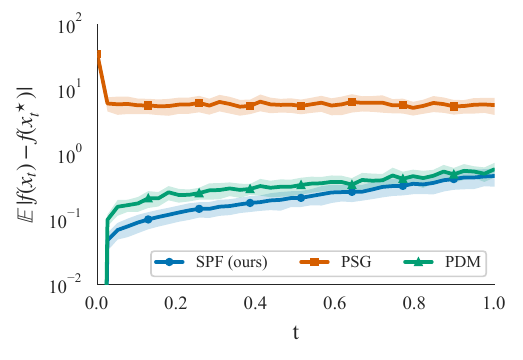}
	}
	\subcaptionbox{\label{subfig:logreg-1}}
	{	
		\hspace{-15pt}
		\includegraphics[width=0.26\textwidth]{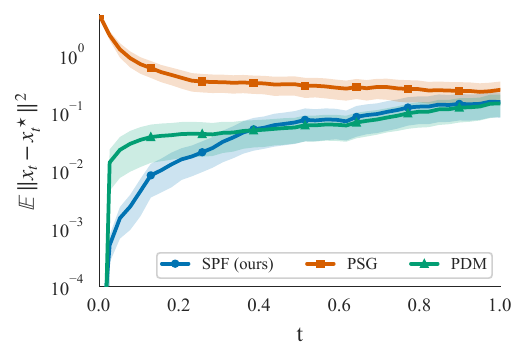}
	}
		\subcaptionbox{\label{subfig:logreg-2}}
	{	
		\hspace{-15pt}
		\includegraphics[width=0.26\textwidth]{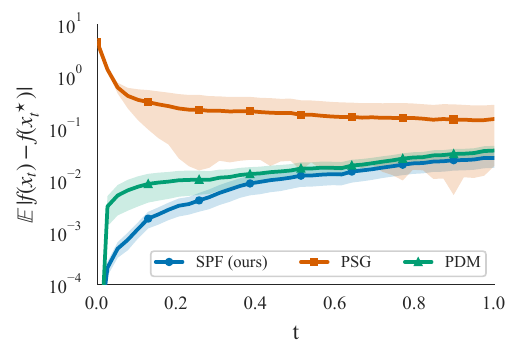}
	}
	\hspace{-15pt}
	\caption{
		\subref{subfig:LS-1} The optimality distance $\mathbb{E}\|\textbf{x}_t-\textbf{x}_t^* \|^2$ 
		and \subref{subfig:LS-2} the objective suboptimality $\mathbb{E}|f(\textbf{x}_t) -  f(\textbf{x}_t^*)|$
		for the considered approaches on the least-squares recovery problem in~\cref{Sec:LS}. 
		\subref{subfig:logreg-1} The optimality distance $\mathbb{E}\|\textbf{x}_t-\textbf{x}_t^* \|^2$ 
		and \subref{subfig:logreg-2} the objective suboptimality $\mathbb{E}|f(\textbf{x}_t) -  f(\textbf{x}_t^*)|$ 
		for the considered approaches on the logistic regression problem in~\cref{Sec:LR}.
	}\label{Fig:LS-LR}
\end{figure*}

Here, we consider a least-squares recovery problem used in~\citep{Cutler2023,Maity2023PredictCorect}. This problem aims to recover the variable $\textbf{x}_t^* \in \mathbb{R}^d$, which follows a non-stationary distribution, from the observed parameter $\textbf{n}_t \in \mathbb{R}^n$ through the following optimization problem:
\begingroup\makeatletter\def\f@size{9.25}\check@mathfonts
\begin{align}
\label{Eq:LS-prob}
\min_{\textbf{X}} \! \underset
	{
	(\textbf{x}_t, \textbf{n}_t \sim \mathcal{P}_t) 
	}
	{\mathbb{E}} 
	\!\left\{ 
	f(\textbf{x}_t, \textbf{n}_t):= \frac12 \big\| \textbf{A} \textbf{x}_t - \textbf{n}_t \big\|^2
	\!\right\},~\forall t\in[0,T]
\end{align}
\endgroup
where $n > d$, $\textbf{X}=(\textbf{x}_t \in \mathbb{R}^d)_{t\in[0,T]}$ is the stochastic process to be determined, $\mathcal{P}_t = \mathcal{N}(\textbf{A}\textbf{x}_t^*, \textbf{C})$ with $\textbf{A} \in \mathbb{R}^{n\times d}$ a fixed measurement matrix, and $\textbf{C} = (\sigma^2/n\|\textbf{A}\|^2)\textbf{I}_n$ the covariance matrix. The hidden parameter $\textbf{x}_t^*$ follows an Ornstein--Uhlenbeck (OU) process:
\begin{align}
\label{Eq:OU-process}
d\textbf{x}_t^* = \boldsymbol{\Theta}(\boldsymbol{\mu} -\textbf{x}_t^*) dt + \boldsymbol{\gamma} dw_t,
\end{align}
where $\boldsymbol{\Theta} \in \mathbb{R}^{n\times n}$, $\boldsymbol{\mu} \in \mathbb{R}^n$, and $\boldsymbol{\gamma} \in \mathbb{R}^n$ are the parameters of the OU process.

We require the gradient $\nabla_x f$ and the Malliavin derivative $D_s\nabla_x f$ for~\cref{Alg:stochastic-path-following}. These are given by $\nabla_x f(\textbf{x}_t,\textbf{n}_t) = \textbf{A}^\top (\textbf{A} \textbf{x}_t - \textbf{n}_t)$ and $D_s\nabla_x f(\textbf{x}_t,\textbf{n}_t) = \textbf{A}^\top (\textbf{A} D_s\textbf{x}_t - D_s\textbf{n}_t)$. Notably, computing $D_s\textbf{n}_t$ does not require knowledge of $\textbf{x}_t^*$. Under the factorization $\boldsymbol{\Theta} = \textbf{V}\textbf{A}$ with $\textbf{V} \in \mathbb{R}^{d\times n}$, we can directly obtain $D_s\textbf{n}_t = D_s\textbf{z}_t$, where $\textbf{z}_t := \textbf{A}\textbf{x}_t^*$ obeys the following SDE:
$$
d\textbf{z}_t = \boldsymbol{\Theta}'(\boldsymbol{\mu}' - \textbf{z}_t) dt + \boldsymbol{\gamma}' dw_t,
$$
with $\boldsymbol{\Theta}' = \textbf{A}\textbf{V}$, $\boldsymbol{\mu}' = \textbf{A}\boldsymbol{\mu}$, and $\boldsymbol{\gamma}' = \textbf{A}\boldsymbol{\gamma}$. Since $n > d$, the factorizing matrix $\textbf{V}$ can be determined as $\textbf{V} = (\textbf{A}^\top \textbf{A})^{-1}\textbf{A}^\top$.

\fakeparagraph%
{Problem and Algorithm Setup}
We primarily follow the setup of~\citep{Cutler2023} for both the problem parameters and the OU-process parameters (i.e., $\boldsymbol{\Theta}$, $\boldsymbol{\mu}$, $\boldsymbol{\sigma}$), all of which have uniform entries; we set $n=10$, $d=5$, $T=1$, and $\sigma=1$ for the time-varying distribution $\mathcal{P}_t$.

We use an NN architecture with three hidden layers, each containing 16 neurons. The SPF algorithm employs ${\rm Gelu}$ activation functions, while the PDM algorithm uses ${\rm LipSwish}$. For PDM, normalization layers are applied after each layer to enhance training stability. We adopt the Adam-type optimizer from~\citep{Adamconvergence2022} with a learning rate of $8\times10^{-3}$. The batch size is set to $p=100$, ensuring an $\epsilon$-stationary point with approximately $\epsilon = 10^{-5}$ precision.

\fakeparagraph%
{Results}
\Cref{subfig:LS-1,subfig:LS-2} compare SPF with the learning-based PDM and gradient-based PSG methods for $\sigma=1$ using two performance metrics: (i) the optimality distance $\mathbb{E}\| \textbf{x}_t - \textbf{x}_t^*\|^2$, which measures how close the recovered solution $\textbf{x}_t$ is to the true target $\textbf{x}_t^*$, and (ii) the objective suboptimality $\mathbb{E}|f(\textbf{x}_t) - f(\textbf{x}_t^*) |$, which measures the difference between the objective evaluated at $\textbf{x}_t$ and the optimum $f(\textbf{x}_t^*)$. Although the PDM algorithm achieves an optimality distance close to that of SPF in~\cref{subfig:LS-1}, SPF outperforms both PDM and PSG overall (see~\cref{Sec:scalability} for a scalability evaluation of SPF and PDM). The same trend holds in~\cref{subfig:LS-2} for objective suboptimality, where SPF exhibits the lowest error among the learning-based and gradient-based approaches.

These results can be explained by the fact that SPF enforces optimality more directly through its loss function, thereby providing a more structured search direction toward optimal solutions.

\subsection{Logistic Regression}
\label{Sec:LR}

\begin{figure*}[t!]
	\centering
	\hspace{-.5em}%
	\subcaptionbox{\label{Fig:scal1}}{%
		\includegraphics[width=0.25\textwidth]{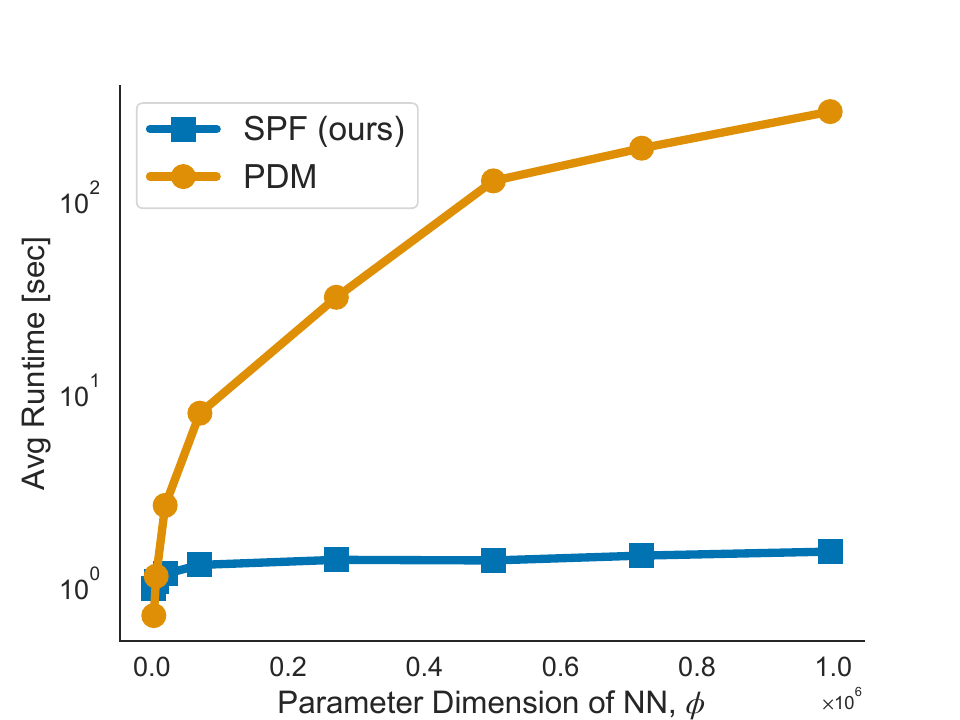}
	}%
	\hspace{8 pt}%
	\subcaptionbox{\label{Fig:scal2}}{%
		\vspace{-10pt}%
		\includegraphics[width=0.27\textwidth]{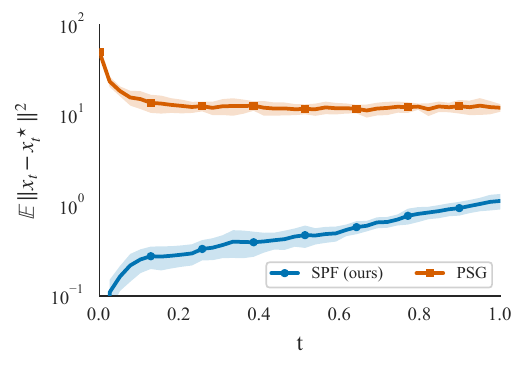}
	}
	\hspace{8 pt}%
	\subcaptionbox{\label{Fig:scal3}}{%
		\vspace{-10pt}%
		\includegraphics[width=0.27\textwidth]{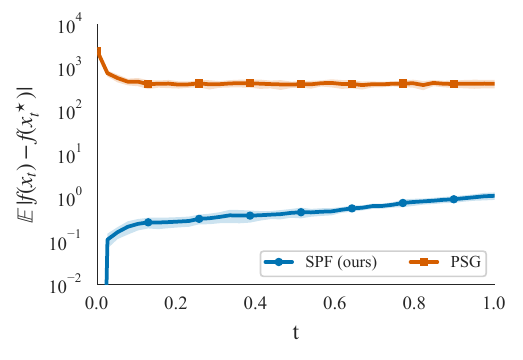}
	}
	\caption{%
		\subref{Fig:scal1} Average runtime on the least-squares recovery problem for SPF and PDM.
		\subref{Fig:scal2}-\subref{Fig:scal3} Performance of SPF on a high-dimensional least-squares recovery task compared with the gradient-based baseline method, PSG.
	}
	\label{Fig:avg-runtime}
\end{figure*}
We now consider an $\ell_2^2$-regularized logistic regression problem used in~\citep{Cutler2023}. Specifically, we study: %
\begingroup\makeatletter\def\f@size{8.75}\check@mathfonts
\begin{align}
	\label{Eq:LR-prob}
	\min_{\textbf{X}} \!\underset{(\textbf{x}_t, \textbf{b}_t)}{\mathbb{E}}\!
	&\bigg\{ 
	f(\textbf{x}_t, \textbf{b}_t) := \frac1n \bigg( \sum_{i=1}^n \log(1 + \exp(\textbf{a}_i^\top \textbf{x}_t)) - \textbf{b}_t^\top \textbf{A} \textbf{x}_t\bigg)\notag\\
	&~~~+\left. \frac{\mu}{2} \| \textbf{x}_t\|^2 
	\right\},\forall t\in[0,T]
\end{align}
\endgroup
where $\textbf{X}=(\textbf{x}_t \in \mathbb{R}^d)_{t\in[0,T]}$ is the stochastic process to be learned, matrix $\textbf{A} \in \mathbb{R}^{n\times d}$ has standard Gaussian rows $\{\textbf{a}_i\}_{i=1}^n$, $\textbf{b}_t \in [0,1]$ is a random sequence of soft labels, and $\mu > 0$ is the regularization coefficient. Moreover, $\textbf{b}_t$ follows a time-varying model $\textbf{b}_{t,i} = 1/(1+e^{-\textbf{y}_{t,i}})$, with $\textbf{y}_t$ obeying the OU process in~\cref{Eq:OU-process}, where $\textbf{y}_{t,i}$ (respectively, $\textbf{b}_{t,i}$) denotes the $i$-th component of $\textbf{y}_t$ (respectively, $\textbf{b}_t$).

Accordingly, we can compute the gradient $\nabla_x f$ and the Malliavin derivative $D_s\nabla_xf$ required to execute~\cref{Alg:stochastic-path-following}. Note that $D_s\nabla_xf$ depends on the Malliavin derivative $D_s\textbf{b}_t$, which can be derived by the chain rule as $D_s\textbf{b}_{t,i} = e^{-\textbf{y}_{t,i}}/\big(1 + e^{-\textbf{y}_{t,i}}\big)^2 D_s\textbf{y}_{t,i}$.

\fakeparagraph%
{Problem and Algorithm Setup}
We follow the same setup as in the previous section with $\mu=10^{-4}$; the SPF algorithm employs ${\rm Tanh}$ activation functions, while the PDM algorithm uses ${\rm Gelu}$ to ensure stability in this experiment. %

\fakeparagraph%
{Results}
\Cref{subfig:logreg-1,subfig:logreg-2} illustrates the performance of SPF compared with the learning-based PDM and gradient-based PSG methods in terms of the optimality distance $\mathbb{E}\| \textbf{x}_t - \textbf{x}_t^*\|^2$ and the objective suboptimality $\mathbb{E}|f(\textbf{x}_t) - f(\textbf{x}_t^*) |$. SPF achieves a better optimality distance than PDM and competitive overall performance. The same pattern is observed for the objective suboptimality in~\cref{subfig:logreg-2}, where SPF improves over the gradient-based PSG baseline and performs comparably to the conventional learning-based PDM method (see %
the next section for a scalability evaluation of SPF and PDM). SPF solves this logistic regression problem effectively by directly enforcing the optimality condition in its formulation.

\section{Scalability}
\label{Sec:scalability}

We finally evaluate the scalability of SPF against the conventional learning-based method (PDM) and the gradient-based baseline (PSG) in a high-dimensional setup.

First, we consider the least-squares recovery problem from~\cref{Sec:LS}, vary the dimension of the NN parameterization, and measure the average time required to update the network parameters for one episode. \Cref{Fig:scal1} presents the average runtime for SPF and PDM. The runtime of SPF remains largely unaffected by the NN parameter dimension~$\phi$, while PDM exhibits a strong dependence on~$\phi$, resulting in significantly higher computational complexity. This difference can be understood from the structural distinction between the two algorithms. SPF simulates the SDEs $\{\textbf{x}_t, \boldsymbol{\Gamma}_{r,t}\}$, which together have total dimension $n + n^2$. Hence, the dimensionality of the required SDEs is insensitive to the parameterization and does not involve differentiation with respect to it, making the SPF forward pass computationally efficient. In contrast, PDM must simulate both the original SDE and its variational counterpart, which together have dimension $n + n\phi$. Consequently, the dimension of the required SDEs scales directly with the parameterization dimension. These factors make PDM inefficient for high-dimensional NNs.

Next, we consider the least-squares problem in~\cref{Eq:LS-prob} in a high-dimensional setting with $n=100$ and $d=50$, while keeping the other configurations identical to those in~\cref{Sec:LS}. \Cref{Fig:scal2,Fig:scal3} compare the performance of SPF with that of the gradient-based PSG method; we do not evaluate PDM owing to its computational inefficiency in this high-dimensional setting. SPF clearly outperforms PSG, demonstrating a substantial performance advantage.

This scalability analysis confirms the effectiveness of the SPF algorithm for high-dimensional problems, both in terms of computational efficiency and overall performance.

\section{Conclusion}%
This work investigated time-varying stochastic optimization (TV-SO), a class of stochastic optimization problems. We derived non-parametric optimality conditions for these problems using Malliavin calculus and proposed a deep-learning algorithm that obtains optimal solutions in a manner insensitive to the parameterization dimension. We compared our method with both learning-based and gradient-based state-of-the-art approaches. The results highlight the advantages of the proposed algorithm in terms of performance and scalability for stochastic optimization problems subject to distributional drift.

Future research could focus on establishing non-parametric sufficient conditions 
and extending the proposed framework to applications such as time-series forecasting.

\section*{Impact statement}

This paper presents work whose goal is to advance the field of machine learning. There are many potential societal consequences of our work, none of which we feel must be specifically highlighted here.

\section*{Acknowledgments}
This work was partially funded by the Research Council of Finland under grants number 354322 and 357533.

\bibliography{tvso}
\bibliographystyle{icml2026}

\appendix
\onecolumn

\section{Preliminaries}

This section introduces preliminary results needed to prove Theorems~\ref{Thm:1} and~\ref{Thm:2}. We begin by recalling Stein’s lemma: for zero-mean Gaussian random variables $\textbf{x}\in\mathbb{R}^n$ and $\textbf{y}\in\mathbb{R}^n$, and a vector-valued,
square-integrable function $\textbf{f}:\mathbb{R}^n \to \mathbb{R}^n$, the following holds:
\begin{align}\label{Eq:SteinLemma}
	\mathbb{E}\left\{ \textbf{f}(\textbf{x})^\top \textbf{y} \right\} 
	= {\rm tr}\left\{ \mathbb{E}\left\{ \mathcal{J}_\textbf{x}\textbf{f}(\textbf{x}) \right\}^\top \mathbb{E}\left\{\textbf{x}\textbf{y}^\top \right\} \right\}
\end{align}

\begin{theorem}\label{Thm:app1}
	Assume that $\textbf{f}:\mathbb{R}^n \to \mathbb{R}^n$ is a vector-valued function in the domain of the Malliavin derivative; $\textbf{X}=(\textbf{x}_t)_{t\in[0,T]}$ is a Wiener-driven, Malliavin-differentiable, and square-integrable process; and $\{ \textbf{G}^l=(\textbf{g}^l_t)_t\}_{l=1}^d$, with $\textbf{g}^l_t \in \mathbb{R}^n$, are adapted and square-integrable processes. Then,
	$$
	\mathbb{E}\left\{ \textbf{f}(\textbf{x}_t) \sum_{l=1}^d \int_0^t \textbf{g}^l_s dw_s^l \right\} 
	= \sum_{l=1}^d \int_0^t \mathbb{E}\left\{ {\textbf{g}^l_s}^\top \mathcal{J}_\textbf{x}\textbf{f}(\textbf{x}_t) D_s^l\textbf{x}_t \right\}ds,
	$$
	where $D_s^l\textbf{x}_t$ is the Malliavin derivative of $\textbf{x}_t$ with respect to the Wiener process $w_s^l$.
	\begin{proof}
		This theorem extends the integration-by-parts formula from Malliavin calculus (see Section~\ref{sec:malliavin-derivative}) to the multivariate setting. Without loss of generality, we prove it for the case $\textbf{x}_t = \int_0^t \sum_{l=1}^d \textbf{h}_s^l dw_s^l$, which generalizes directly to other cases~\citep{nualart2018introduction,alos2021malliavin}. Let $\mathcal{F}^{\textbf{h},\textbf{g}}$ denote the $\sigma$-algebra generated by the processes $\textbf{h}^l$ and $\textbf{g}^l$ up to time $t$. Then,
		\begingroup\makeatletter\def\f@size{9.5}\check@mathfonts
		\begin{align*}
			&\mathbb{E} \left\{ \textbf{f}(\textbf{x}_t) \sum_{l=1}^d \int_0^t \textbf{g}_s^l dw_s^l \right\} 		
			=\mathbb{E}_{\mathcal{F}^{\textbf{h},\textbf{g}} } \mathbb{E}\left\{ \textbf{f}(\textbf{x}_t) \sum_{l=1}^d \int_0^t \textbf{g}_s^l dw_s^l  ~\Big|~ \mathcal{F}^{\textbf{h},\textbf{g}} \right\}\\
			&\hspace{60pt} \overset{(a)}{=} \mathbb{E}_{\mathcal{F}^{\textbf{h},\textbf{g}}}\:{\rm tr}\left\{  
			\sum_{l=1}^d \mathbb{E}\left\{ \mathcal{J}_\textbf{x}\textbf{f} \:\big|\: \mathcal{F}^{\textbf{h},\textbf{g}} \right\} 
			\mathbb{E}\left\{ \int_0^t \textbf{h}_s^l {\textbf{g}_s^l}^\top ds  \:\Big|\: \mathcal{F}^{\textbf{h},\textbf{g}} \right\}\right\}\\
			&\hspace{60pt} = \mathbb{E}_{\mathcal{F}^{\textbf{h},\textbf{g}}}\:\sum_{l=1}^d \int_0^t \mathbb{E}\left\{ {\textbf{g}_s^l}^\top\mathcal{J}_{\textbf{x}}\textbf{f}\:D_s^l\textbf{x}_t \:\Big|\: \mathcal{F}^{\textbf{h},\textbf{g}}\right\} ds
			= \sum_{l=1}^d \int_0^t \mathbb{E}\left\{ {\textbf{g}_s^l}^\top\mathcal{J}_{\textbf{x}}\textbf{f}\: D_s^l\textbf{x}_t \right\} ds
		\end{align*}
		\endgroup
		where (a) uses Stein’s lemma from~\cref{Eq:SteinLemma}, and (b) follows from ${\rm tr}\{\textbf{A}\textbf{z}\textbf{z}^\top\} = {\rm tr}\{\textbf{z}^\top \textbf{A}\textbf{z}\}$. The interchange of integral and expectation is justified, as $F$ is square-integrable and Malliavin-differentiable.
	\end{proof}
\end{theorem}

\section{Proof of Theorem~\ref{Thm:1}}
\label{App:proof1}

We use variational analysis to derive the necessary optimality conditions for Problem $\mbox{TV-SO}$. Let $\boldsymbol{X}=(\textbf{x}_t)_{t\in[0,T]}$ denote an optimal $\mathcal{F}$-adapted process, where $\textbf{x}_t$ depends on the paths of $\{\textbf{w}_s: 0 \leq s \leq t\}$. We therefore consider variations that depend on these paths. To this end, we analyze the variation of the objective $\mathbb{E}\{ f(\textbf{x}_t, \boldsymbol{\xi}_t, t) \}$ by introducing perturbations drawn from the Hilbert space of admissible variations:
\begin{align}\label{Eq:variation}
	\mathcal{V}_t = \bigg\{ \textbf{k}_t + \sum_{l=1}^d \int_0^t \textbf{g}_s^l dw_s^l~\Big|~\textbf{k}_s, \textbf{g}_s^l \in L^2([0,t];\mathbb{R}^n) \bigg\},
\end{align}
where $L^2([0,t];\mathbb{R}^n)$ is the space of $\mathbb{R}^n$-valued square-integrable functions. Here, $\textbf{g}^l_t$ represents a stochastic variation and $\textbf{k}_t$ a deterministic one. We perturb the optimal process $\textbf{X}$ along a direction $\textbf{d}_t \in \mathcal{V}_t$ using
$$
\textbf{x}_t^\epsilon = \textbf{x}_t + \epsilon \:\textbf{d}_t.
$$

Let the variation of $\mathbb{E}\{ f(\textbf{x}_t, \boldsymbol{\xi}_t, t) \}$ along $\textbf{d}$ be denoted as $\delta_\textbf{d}\mathbb{E}f(\textbf{x}_t, \boldsymbol{\xi}_t, t)$. We obtain
\begingroup\makeatletter\def\f@size{9.5}\check@mathfonts
\begin{align}
	\label{Eq:perturb}
	\delta_\textbf{d} \mathbb{E}f(\textbf{x}_t, \boldsymbol{\xi}_t, t) &= \lim_{\epsilon\to 0} 
	\frac{1}{\epsilon} \Big( \mathbb{E}\{ f(\textbf{x}_t^\epsilon, \boldsymbol{\xi}_t, t) \} - \mathbb{E}\{ f(\textbf{x}_t, \boldsymbol{\xi}_t, t) \} \Big)\\
	&= \mathbb{E}\left\{ \nabla_\textbf{x}f(\textbf{x}_t, \boldsymbol{\xi}_t, t)^\top \left( \textbf{k}_t + \sum_{l=1}^d \int_0^t \textbf{g}_s^l dw_s^l \right) \right\} \notag \\
	&\overset{(a)}{=} \mathbb{E}\left\{ \nabla_\textbf{x}f \right\}^\top \textbf{k}_t + \sum_{l=1}^d \int_0^t {\textbf{g}_s^l}^\top \mathbb{E}\left\{ \nabla_{\textbf{x}}^2f D_s\textbf{x}_t^l \right\} ds + \int_0^t {\textbf{g}_s^l}^\top \mathbb{E}\left\{ \nabla_{\textbf{x}\boldsymbol{\xi}}f \:D_s\boldsymbol{\xi}_t^l \right\} ds \notag \\
	&= \mathbb{E}\left\{ \nabla_\textbf{x}f \right\}^\top \textbf{k}_t + \sum_{l=1}^d \int_0^t {\textbf{g}_s^l}^\top \mathbb{E}\left\{ \nabla_{\textbf{x}}^2f D_s\textbf{x}_t^l + \nabla_{\textbf{x}\boldsymbol{\xi}}f \:D_s\boldsymbol{\xi}_t^l \right\} ds. \notag
\end{align}
\endgroup

Here, (a) follows from Theorem~\ref{Thm:app1} and the following identity derived from Stein’s lemma:
$$
\mathbb{E}\left\{ F(\textbf{x},\boldsymbol{\xi})^\top \textbf{y} \right\} 
= {\rm tr}\left\{ \mathbb{E}\left\{ \mathcal{J}_\textbf{x}F(\textbf{x},\boldsymbol{\xi}) \right\}^\top \mathbb{E}\left\{ \textbf{x}\textbf{y}^\top \right\} + \mathbb{E}\left\{ \mathcal{J}_{\boldsymbol{\xi}}F(\textbf{x},\boldsymbol{\xi}) \right\}^\top \mathbb{E}\left\{ \boldsymbol{\xi}\textbf{y}^\top \right\} \right\}.
$$
For $\textbf{x}_t$ to be optimal, $\delta_\textbf{d}\mathbb{E}f(\textbf{x}_t, \boldsymbol{\xi}_t, t) = \textbf{0}$ must hold for all $(\textbf{k}_t,(\textbf{g}_s)_{s\in[0,t]})$ in the Hilbert space $\mathcal{V}_t$. This jointly requires $\mathbb{E}\left\{ \nabla_\textbf{x}f \right\} = \textbf{0}$ and $\mathbb{E}\left\{ \nabla_{\textbf{x}}^2f D_s\textbf{x}_t^l + \nabla_{\textbf{x}\boldsymbol{\xi}}f \:D_s\boldsymbol{\xi}_t^l \right\} = \textbf{0}$ for all $s\in[0,t]$, in accordance with the fundamental lemma of the calculus of variations. Hence, the statement follows.

\section{Proof of Proposition~\ref{Propos:BackwardSDE}}
\label{App2}

According to~\cref{Eq:MalliavinDerv0}, the Malliavin derivative $D_r\textbf{x}_t$—for a given $r$ and as a function of $t$—is a diffusion process adapted to the filtration $\mathcal{F}_r$. Therefore, based on~\cref{Eq:MatrixMalliavin}, the $r$-to-$t$ stochastic flow $\boldsymbol{\Gamma}_{r,t} := \textbf{Y}_t \textbf{Z}_r$, for a given $r$, is adapted to $\mathcal{F}_r$. Moreover, from~\cref{Eq:Y-Z}, we obtain
$$
d\boldsymbol{\Gamma}_{r,t} = d\textbf{Y}_t \textbf{Z}_r 
= \mathcal{J}_\textbf{x}\textbf{a}(\textbf{x}_t,t)\textbf{Y}_t\textbf{Z}_r\,dt
+ \sum_{l=1}^d \mathcal{J}_\textbf{x}\textbf{b}_l(\textbf{x}_t,t)\textbf{Y}_t\textbf{Z}_r\,dw_t^l
= \Bigg( \sum_{l=1}^d \mathcal{J}_\textbf{x}\textbf{b}_l(\textbf{x}_t,t)\,dw_t^l 
+ \mathcal{J}_\textbf{x}\textbf{a}(\textbf{x}_t,t)\,dt \Bigg)\boldsymbol{\Gamma}_{r,t}.
$$
Hence, the statement follows.

\section{Details of the PDM Algorithm}
\label{App:PDM}

We adopt the PDM framework~\citep{tzen2019neural,liu2019neural} as a learning-based baseline for \emph{joint} optimization over the time horizon, i.e., optimization for every $t\in [0,T]$. This approach formulates \emph{parametric} optimality conditions based on a specific diffusion process. Note that the optimality conditions in~\cref{Eq:TV-SO} are non-parametric and do not depend on any particular parameterization policy. However, the $\mbox{TV-SO}$ problem can also be addressed through a parametric reformulation. To that end, we impose a parameterized SDE as in~\cref{Eq:parametricSDE}, so that, given an initial state $\textbf{x}_0$, the full trajectory $(\textbf{x}_t)_{t\in[0,T]}$ is determined by a set of parameters $\boldsymbol{\theta}$. The $\mbox{TV-SO}$ problem can then be solved via
\begin{align}\label{Eq:parametricTVSO}
	&\min_{\boldsymbol{\theta}}~~\mathbb{E}\big\{ f(\textbf{x}_t,\boldsymbol{\xi}_t,t) \big\},\quad \forall\, t\in [0,T], \notag \\
	& ~~ \text{s.t.}~~~ d\textbf{x}_t = \textbf{a}(\textbf{x}_t, t, \boldsymbol{\theta})\,dt
	+ \sum_{l=1}^d\textbf{b}_l(\textbf{x}_t, t, \boldsymbol{\theta})\,d\textbf{w}_t^l.
\end{align}

Taking the derivative with respect to $\boldsymbol{\theta}$ yields the condition
\[
\mathbb{E}\big\{ \nabla_\textbf{x}f(\textbf{x}_t,\boldsymbol{\xi}_t,t)^\top J_t^{\boldsymbol{\theta}} \big\}=\textbf{0},
\]
which must hold for every $t$. Here, $J_t^{\boldsymbol{\theta}} = \mathcal{J}_{\boldsymbol{\theta}}\textbf{x}_t$ is the pathwise Jacobian of $\textbf{x}_t$ with respect to $\boldsymbol{\theta}$ and satisfies the variational SDE
\begingroup\makeatletter\def\f@size{9}\check@mathfonts
\begin{align}\label{Eq:variationlSDE}
	d\boldsymbol{J}_t^{\boldsymbol{\theta}} = \big(\mathcal{J}_{\boldsymbol{\theta}} \textbf{a}(\textbf{x}_t,t,\boldsymbol{\theta}) + 
	\mathcal{J}_\textbf{x} \textbf{a}(\textbf{x}_t,t,\boldsymbol{\theta}) \boldsymbol{J}_t^{\boldsymbol{\theta}} \big) dt 
	+ \sum_{l=1}^d \big(\mathcal{J}_{\boldsymbol{\theta}}\textbf{b}_l(\textbf{x}_t, t,\boldsymbol{\theta})
	+\mathcal{J}_\textbf{x}\textbf{b}_l(\textbf{x}_t, t, \boldsymbol{\theta}) \boldsymbol{J}_t^{\boldsymbol{\theta}} \big) d\textbf{w}_t^l,
\end{align}
\endgroup
with the initial condition $\boldsymbol{J}_0^{\boldsymbol{\theta}}=\textbf{0}$. Since the sign of $f(\textbf{x}_t,\boldsymbol{\xi}_t,t)$ may vary with $t$, and the condition $\mathbb{E}\big\{ \nabla_\textbf{x}f^\top \mathcal{J}_{\boldsymbol{\theta}}\textbf{x}_t \big\}=\textbf{0}$ should hold jointly for all $t\in[0,T]$, the corresponding parametric loss is defined as
\begingroup\makeatletter\def\f@size{8.5}\check@mathfonts
\begin{align}\label{Eq:ParametricLoss}
	\!\!\!\!\mathcal{L}_{\rm bench}(\boldsymbol{\theta}) =\! \int_0^T\!\! \Big\| \frac{d}{dt} \mathbb{E}\{ \nabla_\textbf{x}f^\top \boldsymbol{J}_t^{\boldsymbol{\theta}}  \} \Big\|^2 dt 
	=\! \int_0^T\!\! \bigg\| \frac{1}{dt}\mathbb{E}
	\Big\{ 
	d\nabla_\textbf{x}f^\top \boldsymbol{J}_t^{\boldsymbol{\theta}}+
	\nabla_\textbf{x}f^\top d\boldsymbol{J}_t^{\boldsymbol{\theta}} +
	d\nabla_\textbf{x}f^\top d\boldsymbol{J}_t^{\boldsymbol{\theta}} 
	\Big\} \bigg\|^2 dt,
\end{align}
\endgroup
where the differentials inside the norm can be simplified using Itô’s formula. Note that the loss in~\cref{Eq:ParametricLoss} serves as the stochastic analogue of the deterministic loss in~\cref{eq:non-parametric-loss}, differing in that it explicitly depends on the parameters of the SDE in~\cref{Eq:parametricSDE}, whereas the losses in~\cref{Eq:crit2} and~\cref{Eq:crit3} do not. Accordingly, the PDM algorithm (referred to simply as PDM) proceeds as follows:
\begin{enumerate}
	\item \textbf{Forward pass:} jointly simulate the parametric SDE in~\cref{Eq:parametricSDE} and the variational SDE in~\cref{Eq:variationlSDE} under the current $\boldsymbol{\theta}$.
	\item \textbf{Loss evaluation:} compute the parametric loss $\mathcal{L}_{\rm bench}(\boldsymbol{\theta})$ using~\cref{Eq:ParametricLoss}.
	\item \textbf{Parameter update:} update the drift and diffusion parameters based on the computed loss.
\end{enumerate}

\section{Algorithm Setup for Baselines}
\label{App:toy-setup}
We use the same neural network structure for the SPF (respectively, PDM) algorithm. To this end, we consider a deep-learning architecture with $3$ hidden layers, each with $32$ neurons, equipped with normalization layers and activation functions ${\rm LipSwish(\cdot)}$ (respectively, ${\rm tanh(\cdot)}$). We employ the Adam optimizer with a learning rate in the range $[8\times10^{-3},10^{-2}]$. The number of batches, which simulate realizations of the diffusion process, is set to $400$.

\section{Functional Stochastic Optimization}
\label{App:FSO}

We next consider functional stochastic optimization (F-SO) problems, defined as follows:
\begin{equation}
	\begin{aligned}\label{Eq:F-SO}
		\mbox{F-SO}: \qquad
		&\min_{ \textbf{X} }  
		\!\!
		\underset{
			\begin{tiny} 
				\begin{array}{c}
					\textbf{X},\textbf{X}'\!\sim\!\mathbb{Q}\\
					\boldsymbol{\Xi},\boldsymbol{\Xi}'\!\sim\!\mathbb{P}
				\end{array} 
			\end{tiny} 
		}{\mathbb{E}}
		\int_0^T	f(\textbf{x}_t, \textbf{x}_t', \boldsymbol{\xi}_t,\boldsymbol{\xi}_t',t) dt,
	\end{aligned}
	\vspace{-5pt}
\end{equation}
where:
\begin{itemize}
	\item $\textbf{X}$ and $\boldsymbol{\Xi}$ are $\mathcal{F}$-adapted, Malliavin-differentiable processes.
	\item $\textbf{x}_t$ (respectively, $\boldsymbol{\xi}_t$) is an independent copy of $\textbf{x}_t'$ (respectively, $\boldsymbol{\xi}_t'$), i.e., the processes $\textbf{X}=(\textbf{x}_t)_{t\in[0,T]}$ and $\textbf{X}'=(\textbf{x}'_t)_{t\in[0,T]}$ (respectively, $\boldsymbol{\Xi}=(\boldsymbol{\xi}_t)_{t\in[0,T]}$ and $\boldsymbol{\Xi}'=(\boldsymbol{\xi}_t')_{t\in[0,T]}$) share the same law $\mathbb{Q}$ (respectively, $\mathbb{P}$) but arise from independent Wiener processes $\textbf{W}$ and $\textbf{W}'$.
	\item The function $f:\mathbb{R}^n\!\times\!\mathbb{R}^n\!\times\!\mathbb{R}^m\!\times\!\mathbb{R}^m\to\mathbb{R}$ is twice continuously differentiable with respect to its arguments, square-integrable, and symmetric in the pairs $(\textbf{x}_t,\textbf{x}_t')$ and $(\boldsymbol{\xi}_t,\boldsymbol{\xi}_t')$.
\end{itemize} 

The following theorem specifies \emph{non-parametric} optimality conditions for Problem $\mbox{F-SO}$.

\begin{theorem}
	\label{Thm:2}
	Consider Problem $\mbox{F-SO}$, where $\textbf{X}=(\textbf{x}_t)_{t\in[0,T]}$ and $\boldsymbol{\Xi}=(\boldsymbol{\xi}_t)_{t\in[0,T]}$ are two $\mathcal{F}$-adapted, Malliavin-differentiable processes driven by the Wiener process $\textbf{W}$. Moreover, $\textbf{x}_t'$ (respectively, $\boldsymbol{\xi}_t'$) is an independent copy of $\textbf{x}_t$ (respectively, $\boldsymbol{\xi}_t$). Then Problem $\mbox{F-SO}$ can be optimized by enforcing the following optimality conditions:
	\begin{align}\label{Eq:Conditions_fSO}
		&\mathbb{E}\bigg\{ 
		\int_s^T \!\!
		\big(
		\nabla_\textbf{x}^2 f D_s^i \textbf{x}_t + 
		\nabla_{\textbf{x}\boldsymbol{\xi}} f  D_s^i \boldsymbol{\xi}_t
		\big) \:dt
		\bigg\} = \textbf{0}, \notag \\
		&\mathbb{E}\big\{ \nabla_{\textbf{x}}f \big\} = \textbf{0},\qquad s\in[0,T]
	\end{align}
	for $i \in\{1,...,d\}$ and $s\in[0,T]$, where $f$ abbreviates $f(\textbf{x}_t, \textbf{x}_t',\boldsymbol{\xi}_t, \boldsymbol{\xi}_t',t)$, and $D_s^i \textbf{x}_t$ (respectively, $D_s^i \boldsymbol{\xi}_t$) denotes the Malliavin derivative of $\textbf{x}_t$ (respectively, $\boldsymbol{\xi}_t$) with respect to $w_t^i$.
\end{theorem}

\begin{proof}
	As in Appendix~\ref{App:proof1}, we use the variational Hilbert space $\mathcal{V}_t$ from~\cref{Eq:variation} to compute the variation of the objective $\int_0^T f(\textbf{x}_t, \textbf{x}_t', \boldsymbol{\xi}_t,\boldsymbol{\xi}_t',t)\,dt$. The processes $\textbf{X}$ and $\textbf{X}'$ share the same dynamics but are driven by independent Wiener processes $\textbf{W}$ and $\textbf{W}'$. Accordingly, we consider the variational directions $\textbf{d}_t = \textbf{k}_t + \sum_{l=1}^d \int_0^t \textbf{g}_s^l dw_s^l$ and $\textbf{d}_t' = \textbf{k}_t + \sum_{l=1}^d \int_0^t \textbf{g}_s^l d{w_s^l}'$ for $\textbf{x}_t$ and $\textbf{x}_t'$, respectively. Let $D_s^i\textbf{x}_t'$ and $D_s^i\boldsymbol{\xi}_t'$ denote the Malliavin derivatives of $\textbf{x}_t'$ and $\boldsymbol{\xi}_t'$ with respect to ${w_s^i}'$. Then,
	\begingroup\makeatletter\def\f@size{9.}\check@mathfonts
	\begin{align*}
		\delta_\textbf{d} \mathbb{E}\int_0^T\!\! f&(\textbf{x}_t, \textbf{x}_t', \boldsymbol{\xi}_t,\boldsymbol{\xi}_t',t) dt
		= \int_0^T\!\! \mathbb{E}\left\{ \nabla_\textbf{x}f^\top \left( \textbf{k}_t + \sum_{l=1}^d \int_0^t\textbf{g}_s^l dw_s^l \right) 
		+ \nabla_{\textbf{x}'}f^\top \left( \textbf{k}_t + \sum_{l=1}^d \int_0^t\textbf{g}_s^l d{w_s^l}' \right)
		\right\} dt\\
		\overset{(a)}{=}& \int_0^T\!\! \mathbb{E}\left\{ \nabla_\textbf{x}f \right\}^\top \textbf{k}_t dt 
		+ \int_0^T\!\!\mathbb{E}\left\{ \nabla_{\textbf{x}'}f \right\}^\top \textbf{k}_t dt
		+ \sum_{l=1}^d \int_0^T\!\!\int_0^t {\textbf{g}_s^l}^\top \mathbb{E}\left\{ \nabla_{\textbf{x}}^2f D_s\textbf{x}_t^l 										 				+\nabla_{\textbf{x}\boldsymbol{\xi}}f \:D_s\boldsymbol{\xi}_t^l  \right\} ds\: dt \\
		&+ \sum_{l=1}^d \int_0^T\!\!\int_0^t {\textbf{g}_s^l}^\top \mathbb{E}\left\{ \nabla_{\textbf{x}'}^2f D_s{\textbf{x}_t^l}'
		+\nabla_{\textbf{x}'\boldsymbol{\xi}'}f \:D_s{\boldsymbol{\xi}_t^l}'  \right\} ds\: dt \\
		\overset{(b)}{=}& \int_0^T\!\!\mathbb{E}\left\{ \nabla_\textbf{x}f \right\}^\top \textbf{k}_t dt
		+ \int_0^T\!\! \mathbb{E}\left\{ \nabla_{\textbf{x}'}f \right\}^\top \textbf{k}_t dt
		+ \sum_{l=1}^d \int_0^T\!\! {\textbf{g}_s^l}^\top  \int_s^T \mathbb{E}\left\{ \nabla_{\textbf{x}}^2f D_s\textbf{x}_t^l 		+\nabla_{\textbf{x}\boldsymbol{\xi}}f \:D_s\boldsymbol{\xi}_t^l  \right\} dt\: ds \\
		&+ \sum_{l=1}^d \int_0^T\!\! {\textbf{g}_s^l}^\top \int_s^T  \mathbb{E}\left\{ \nabla_{\textbf{x}'}^2f D_s{\textbf{x}_t^l}'
		+\nabla_{\textbf{x}'\boldsymbol{\xi}'}f \:D_s{\boldsymbol{\xi}_t^l}'  \right\} dt\: ds,
	\end{align*}
	\endgroup
	where (a) uses \cref{Thm:app1}, and (b) follows from changing the order of integration.
	
	For $(\textbf{x}_t,\textbf{x}_t')$ to be optimal, we require $\delta_\textbf{d}\mathbb{E}\int_0^T f(\textbf{x}_t, \textbf{x}_t',\boldsymbol{\xi}_t,\boldsymbol{\xi}_t',t) dt= \textbf{0}$ for all $(\textbf{k}_t,(\textbf{g}_s)_{s\in[0,t]})$ in the Hilbert space $\mathcal{V}_t$. By the fundamental lemma of the calculus of variations and the symmetry of $f$ in $(\textbf{x}_t,\textbf{x}_t')$ and $(\boldsymbol{\xi}_t,\boldsymbol{\xi}_t')$, this jointly implies $\mathbb{E}\left\{ \nabla_\textbf{x}f \right\} = \textbf{0}$ for $t\in[0,T]$ and $\int_s^T \mathbb{E}\left\{ \nabla_{\textbf{x}}^2f D_s\textbf{x}_t^l + \nabla_{\textbf{x}\boldsymbol{\xi}}f \:D_s\boldsymbol{\xi}_t^l \right\}\:dt = \textbf{0}$ for $s\in[0,T]$, which yields the stated conditions.
\end{proof}

\begin{remark}
	Problem $\mbox{F-SO}$ can be optimized using the following energy-functional criterion:
	\begin{align}\label{Eq:crit3}
		\mathcal{L}_{F}(\textbf{x},D\textbf{x}) =& \sum_{i=1}^d \! \int_0^T \! 
		\Big\| \int_s^T \! \mathbb{E}
		\big\{
		\nabla_\textbf{x}^2 f \: D_s^i \textbf{x}_t 
		+\nabla_{\textbf{x}\boldsymbol{\xi}} f \: D_s^i \boldsymbol{\xi}_t 
		\big\}  dt 
		\Big\|^2 ds 
		+  \int_0^T \Big\|\mathbb{E}\big\{ \nabla_{\textbf{x}}f \big\} \Big\|^2 \: dt,
	\end{align}
	where $f$ is shorthand for $f(\textbf{x}_t, \textbf{x}_t',\boldsymbol{\xi}_t, \boldsymbol{\xi}_t',t)$.
\end{remark}

\subsection{Probability Law Following}
\label{App:plf}

In this section, we aim to construct a diffusion process whose marginal trajectory matches a given non-stationary distribution. Langevin dynamics~\citep{Brosse:2018:NIPS} can be applied in the case of stationary laws. Let $(\mathbb{P}(\cdot,t))_{t \in[0,T]}$ be a family of probability laws. We seek a diffusion process $(\textbf{x}_t)_{t \in [0,T]}$ inducing the law $\mathbb{Q}(\cdot,t)=\text{Pr}(\textbf{x}_t\in\cdot)$ such that $\mathbb{Q}(\textbf{x}_t,t) = \mathbb{P}(\textbf{x}_t,t)$ for almost every $t\in[0,T]$.

We employ the kernelized Stein discrepancy (KSD)~\citep{Liu2016} to enforce similarity between the distributions $\mathbb{P}$ and $\mathbb{Q}$. This discrepancy uses a smooth, positive definite kernel $K(\textbf{x}, \textbf{x}')$ and is defined as
$
S(\mathbb{Q},\mathbb{P}) = E_{\textbf{x},\textbf{x}'\sim \mathbb{Q}}\{ \kappa_\mathbb{P}(\textbf{x},\textbf{x}') \},
$
where
$$
\kappa_\mathbb{P}(\textbf{x},\textbf{x}') 
= \textbf{s}_\mathbb{P}(\textbf{x})^\top K(\textbf{x},\textbf{x}')\textbf{s}_\mathbb{P}(\textbf{x}') 
+ \textbf{s}_\mathbb{P}(\textbf{x})^\top \nabla_{\textbf{x}'}K(\textbf{x},\textbf{x}')
+ \nabla_{\textbf{x}}K(\textbf{x},\textbf{x}')^\top \textbf{s}_\mathbb{P}(\textbf{x}') 
+ {\rm tr}(\nabla_{\textbf{x},\textbf{x}'}K(\textbf{x},\textbf{x}') ),
$$
and $\textbf{s}_\mathbb{P}(\textbf{x})$ is the score function of $\mathbb{P}$, i.e., $\textbf{s}_\mathbb{P}(\textbf{x})= \nabla_\textbf{x} \log \mathbb{P}(\textbf{x})$.
We then formulate an $\mbox{F-SO}$ problem that uses the average KSD as the objective and minimize it over the diffusion process in~\cref{Eq:diffusion}:
\begingroup\makeatletter\def\f@size{9.}\check@mathfonts
\begin{align}\label{Eq:Exp3}
	\min_{{\rm Diff}^{\textbf{a},\textbf{B}}} E_{\textbf{x},\textbf{x}'\sim \mathcal{\mathbb{Q}}} ~\frac1T \int_0^T \!\kappa_{\mathbb{P}(\cdot,t)}(\textbf{x}_t,\textbf{x}_t') dt.
\end{align}
\endgroup
To minimize this objective, we use the optimality condition $\mathcal{L}_F$ in~\cref{Eq:crit3} and apply the SPF algorithm with learning loss $\mathcal{L}_F$ to enforce alignment between the induced and target laws.

\noindent \textbf{Experiment:}
We now consider several target distributions, with $\mathbb{P}_i(\cdot,t) = 0.3 \mathcal{N}(\mu_{i,1}(t), \sigma_{i,1}(t)^2) + 0.7 \mathcal{N}(\mu_{i,2}(t), \sigma_{i,2}(t)^2)$ for $i\in\{1,2\}$, where
$\boldsymbol{\mu}_1(t) = [4t,-4t]$,
$\boldsymbol{\mu}_2(t) = [4-4t,4t-4]$,
$\boldsymbol{\sigma}_1^2(t) = \boldsymbol{\sigma}_2^2(t) = [0.5+t,0.5+2t]$.
We also consider $\mathbb{P}_i(\cdot,t) = \mathcal{N}(0, \sigma_{i}(t)^2)$ for $i\in\{3,4\}$, where $\sigma_3^2(t) = 0.5+3.5t$ and $\sigma_4^2(t) = 4-3.5t$, and $\mathbb{P}_5(\cdot,t) = \mathcal{N}(4t, 1)$, all for $t\in[0,1]$.

For comparison, we employ the Finite Dimensional Matching (FDM) and Truncated Signature kernel (TruncSig) approaches~\citep{Bonnier2024truncSig,Zhang2025FDM} as learning-based baselines. 
These methods use neural SDEs together with scoring rules designed to learn probability laws of diffusion processes. 
FDM employs a rule tailored to Markov processes, whereas TruncSig uses a rule based on truncated signature kernels~\citep{Zhang2025FDM}. We apply these approaches to solve the F-SO problem in~\cref{Eq:Exp3} and compare their performance with that of SPF.

\noindent \textbf{Settings:}
For this experiment, we use a Gaussian kernel
$$
K(\textbf{x},\textbf{x}') = \exp\left\{  \| \textbf{x}-\textbf{x}'\|^2/(2\sigma^2) \right\},
$$
with bandwidth $\sigma=1$~\citep{Zhang2025FDM}. The neural architecture consists of $4$ hidden layers, each with $32$ neurons and ${\rm tanh(\cdot)}$ activations. We employ Adam\,/\,AdamW optimizers with learning rates in $[6\times10^{-4},10^{-3}]$, and set the number of batches to $500$. For FDM and TruncSig, we use the respective scoring rules proposed in~\citep{Zhang2025FDM,Bonnier2024truncSig} as training losses for learning the probability laws.

\noindent \textbf{Results:}
\begin{table*}[t!]
	\centering
	\setlength{\tabcolsep}{14pt} 
	\caption{
		\textbf{Average KSD (mean $\pm$ std) for the law-following problems with $\mathbb{P}_i$ defined in Section~\cref{App:plf}.}
		Comparison of our SPF algorithm against SOTA methods (FDM and TruncSig).
		SPF outperforms SOTA for most of the considered laws $\mathbb{P}_i$.}
	\label{Tab2:quantative-results}
	\small~\\[.5em]
	\begin{tabular}{@{} l  c c c c c @{}}
		\toprule
		\textbf{Model}       &  $\mathbb{P}_1$ & $\mathbb{P}_2$ & $\mathbb{P}_3$ & $\mathbb{P}_4$ & $\mathbb{P}_5$\\
		\midrule
		\textbf{SPF} (\textbf{ours}) & \textbf{0.04 $\pm$ 0.01} & \textbf{0.037 $\pm$ 0.004} & \textbf{0.02 $\pm$ 0.005} & 0.018 $\pm$ 0.005 & \textbf{0.011 $\pm$ 0.003}\\
		\textbf{FDM}         & 0.08 $\pm$ 0.015 & 0.051 $\pm$ 0.008 & 0.028 $\pm$ 0.015 & \textbf{0.016 $\pm$ 0.003} & 0.015 $\pm$ 0.005 \\
		\textbf{TruncSig}    & 0.08 $\pm$ 0.01 & 0.16 $\pm$ 0.022 & 0.07 $\pm$ 0.015 & 0.22 $\pm$ 0.05 & 0.12 $\pm$ 0.02\\
		\bottomrule
	\end{tabular}
\end{table*}
Tab.~\ref{Tab2:quantative-results} reports the performance of SPF relative to the state-of-the-art methods FDM and TruncSig in terms of average KSD (mean $\pm$ std) for the law-following problems with $\mathbb{P}=\mathbb{P}_i$, $i\in\{1,\dots,5\}$. The \emph{mean} and \emph{std} of the average KSD are computed from multiple episodes after convergence. Our method outperforms these baselines for most target laws and achieves comparable performance for $\mathbb{P}_4$. SPF attains these results by directly enforcing the optimality condition through its loss function, thereby providing a more structured search direction toward optimal solutions.

\end{document}

%% file: math_commands.tex
\usepackage{amsmath,amsfonts,bm}

\def\eqref#1{equation~\ref{#1}}

\def\1{\bm{1}}

\DeclareMathAlphabet{\mathsfit}{\encodingdefault}{\sfdefault}{m}{sl}
\SetMathAlphabet{\mathsfit}{bold}{\encodingdefault}{\sfdefault}{bx}{n}